\documentclass{article}

\usepackage[accepted]{icml2021}

\usepackage[utf8]{inputenc} % allow utf-8 input
\usepackage[T1]{fontenc}    % use 8-bit T1 fonts
\usepackage{hyperref}       % hyperlinks
\usepackage{xurl}            % simple URL typesetting
\usepackage{booktabs}       % professional-quality tables
\usepackage{amsfonts}       % blackboard math symbols
\usepackage{nicefrac}       % compact symbols for 1/2, etc.
\usepackage{microtype}      % microtypography
\usepackage{caption}
\usepackage{amssymb}
\usepackage{graphicx}
\usepackage{color}
\usepackage{complexity}
\usepackage{subcaption}
\usepackage{amsmath}
\usepackage{amsthm}
\usepackage{pifont}
\usepackage{tabularx}
\usepackage{fancyvrb}
\usepackage{comment}
\usepackage{algorithm}
\usepackage{algorithmic}
\usepackage{float}
\usepackage{wrapfig}
\usepackage{mathtools}
\usepackage{thmtools,thm-restate}
\usepackage{paralist} 
\usepackage{multirow}
\usepackage{scrextend}
\usepackage{enumitem}
\usepackage[noorphans,vskip=0ex]{quoting}
\usepackage{bm}
\usepackage{bbm}
\usepackage{thm-restate}
\usepackage{tabularx}
\usepackage{xcolor}
\usepackage{xspace}
\usepackage{adjustbox}
\usepackage{pgf,tikz,pgfplots}
\pgfplotsset{compat=1.15}
\usepackage{mathrsfs}
\usetikzlibrary{arrows}
\theoremstyle{definition}
\newtheorem{definition}{Definition}[section]

\newtheorem{assumption}{Assumption}[section]

\renewcommand{\algorithmicrequire}{\textbf{Input:}}
\renewcommand{\algorithmicensure}{\textbf{Output:}}

\newcommand{\dist}{\mathcal{D}}
\newcommand{\distp}{\mathcal{P}}
\newcommand{\distq}{\mathcal{Q}}

\newcommand{\xxspace}{\mathcal{X}}
\newcommand{\yyspace}{\mathcal{Y}}
\newcommand{\zzspace}{\mathcal{Z}}
\newcommand{\aaspace}{\mathcal{A}}
\newcommand{\hhspace}{\mathcal{H}}
\newcommand{\Ypred}{\widehat{Y}}
\newcommand{\Apred}{\widehat{A}}

\newcommand{\err}{\mathrm{Err}}

\newcommand{\RR}{\mathbb{R}}

\newcommand{\Exp}{\mathbb{E}}
\newcommand{\HH}{\mathcal{H}}
\newcommand{\xx}{\mathbf{x}}
\newcommand{\zz}{\mathbf{z}}

\newcommand*\dif{\mathop{}\!\mathrm{d}}

\DeclareMathOperator*{\argmin}{arg\,min}

\newcommand{\defeq}{\vcentcolon=}

\newcommand{\ind}{\mathbb{I}}
\DeclarePairedDelimiterX{\inp}[2]{\langle}{\rangle}{#1, #2}

\newcommand{\kl}{D_{\text{KL}}}
\newcommand{\dtv}{d_{\text{TV}}}

\newcommand{\errgap}{\Delta_{\mathrm{Err}}}
\newcommand{\crossentropy}[3]{\mathrm{CE}_{#1}(#2~\|~#3)}

\newcommand{\mse}[3]{\mathrm{MSE}_{#1}(#2,~#3)}
\newcommand{\ie}{\textit{i}.\textit{e}.}
\newcommand{\eg}{\textit{e}.\textit{g}.}

\definecolor{dkgreen}{rgb}{0,0.6,0}
\definecolor{gray}{rgb}{0.5,0.5,0.5}
\definecolor{mauve}{rgb}{0.58,0,0.82}
\setlist{nolistsep}

\icmltitlerunning{Understanding and Mitigating Accuracy Disparity in Regression}

\begin{document}

\twocolumn[
\icmltitle{Understanding and Mitigating Accuracy Disparity in Regression}



\icmlsetsymbol{equal}{*}

\begin{icmlauthorlist}
\icmlauthor{Jianfeng Chi}{uva}
\icmlauthor{Yuan Tian}{uva}
\icmlauthor{Geoffrey J. Gordon}{cmu}
\icmlauthor{Han Zhao}{uiuc}
\end{icmlauthorlist}

\icmlaffiliation{uva}{Department of Computer Science, University of Virginia}
\icmlaffiliation{cmu}{Machine Learning Department, Carnegie Mellon University}
\icmlaffiliation{uiuc}{Department of Computer Science, University of Illinois at Urbana-Champaign}

\icmlcorrespondingauthor{Jianfeng Chi}{jc6ub@virginia.com}
\icmlcorrespondingauthor{Han Zhao}{hanzhao@illinois.edu}

\icmlkeywords{Algorithmic Fairness, Regression}

\vskip 0.3in]



\printAffiliationsAndNotice{} % otherwise use the standard text.

\begin{abstract}
With the widespread deployment of large-scale prediction systems in high-stakes domains, \eg, face recognition, criminal justice, etc., disparity in prediction accuracy between different demographic subgroups has called for fundamental understanding on the source of such disparity and algorithmic intervention to mitigate it. In this paper, we study the accuracy disparity problem in regression. To begin with, we first propose an \emph{error decomposition theorem}, which decomposes the accuracy disparity into the distance between marginal label distributions and the distance between conditional representations, to help explain why such accuracy disparity appears in practice. Motivated by this error decomposition and the general idea of distribution alignment with statistical distances, we then propose an algorithm to reduce this disparity, and analyze its game-theoretic optima of the proposed objective functions. To corroborate our theoretical findings, we also conduct experiments on five benchmark datasets. The experimental results suggest that our proposed algorithms can effectively mitigate accuracy disparity while maintaining the predictive power of the regression models. 
\end{abstract}

\section{Introduction}

Recent progress in machine learning has led to its widespread use in many high-stakes domains, such as criminal justice, healthcare, student loan approval, and hiring. Meanwhile, it has also been widely observed that accuracy disparity could occur inadvertently under various scenarios in practice~\citep{barocas2016big}. For example, errors are inclined to occur for individuals of certain underrepresented demographic groups~\citep{kim2016data}. In other cases, \citet{buolamwini2018gender} showed that notable accuracy disparity exists across different racial and gender demographic subgroups on several real-world image classification systems. Moreover, \citet{bagdasaryan2019differential} found out that a differentially private model even exacerbates such accuracy disparity. Such accuracy disparity across demographic subgroups not only raises concerns in high-stake applications but also can be utilized by malicious parties to cause information leakage~\citep{yaghini2019disparate,zhao2020trade}. 
% These phenomena raise the concern: ``accurate, but for whom?''~\citep{chouldechova2017fairer}.

Despite the ample needs of accuracy parity, most prior work limits its scope to studying the problem in binary classification settings~\citep{hardt2016equality,zafar2017fairness-b,zhao2019inherent,jiang2019wasserstein}. 
Compared to the accuracy disparity problem in classification settings, accuracy disparity\footnote{Technically, accuracy disparity refers to (squared) error difference in our paper. We would like to use accuracy disparity throughout our paper since it is a more commonly used term in fairness problems.} in regression is a more challenging but less studied problem, due to the fact that many existing algorithmic techniques designed for classification cannot be extended in a straightforward way when the target variable is continuous~\citep{zhao2019conditional}.
In a seminal work, \citet{chen2018my} analyzed the impact of data collection on accuracy disparity in general learning models. They provided a descriptive analysis of such parity gaps and advocated for collecting more training examples and introducing more predictive variables. While such a suggestion is feasible in applications where data collection and labeling is cheap, it is not applicable in domains where it is time-consuming, expensive, or even infeasible to collect more data, \eg, in autonomous driving, education, etc. 

\paragraph{Our Contributions}
In this paper, we provide a prescriptive analysis of accuracy disparity and aim at providing algorithmic interventions to reduce the disparity gap between different demographic subgroups in the regression setting. To start with, we first formally characterize why accuracy disparity appears in regression problems by depicting the feasible region of the underlying group-wise errors. Next, we derive an error decomposition theorem that decomposes the accuracy disparity into the distance between marginal label distributions and the distance between conditional representations. 
We also provide a lower bound on the joint error across groups. 
Based on these results, we illustrate why regression models aiming to minimize the global loss will inevitably lead to accuracy disparity if the marginal label distributions or conditional representations differ across groups. See Figure~\ref{fig:geo} for illustration.

Motivated by the error decomposition theorem, we propose two algorithms to reduce accuracy disparity via joint distribution alignment with the total variation distance and the Wasserstein distance, respectively. Furthermore, we analyze the game-theoretic optima of the objective functions and illustrate the principle of our algorithms from a game-theoretic perspective. To corroborate the effectiveness of our proposed algorithms in reducing accuracy disparity, we conduct experiments on five benchmark datasets.
% \footnote{Our code is publicly available at:\\\url{https://github.com/JFChi/Understanding-and-Mitigating-Accuracy-Disparity-in-Regression}}
Experimental results suggest that our proposed algorithms help to mitigate accuracy disparity while maintaining the predictive power of the regression models.
We believe our theoretical results contribute to the understanding of why accuracy disparity occurs in machine learning models, and the proposed algorithms provides an alternative for intervention in real-world scenarios where accuracy parity is desired but collecting more data/features is time-consuming or infeasible.

% \begin{figure*}[!htb]
% \centering
% % \vspace{-4mm}
% \begin{subfigure}[b]{.48\linewidth}
%   \centering
%   \includegraphics[width=\linewidth]{figures/geo_intepretation.pdf}
%     \caption{Geometric interpretation of accuracy disparity.}
%   \label{fig:geo}
% \end{subfigure}
% ~
% \begin{subfigure}[b]{.48\linewidth}
%   \centering
%   \includegraphics[width=\linewidth]{figures/game_theory_figure.pdf}

%   \caption{Game-theoretic illustration of our algorithms.}
%   \label{fig:game}
% \end{subfigure}
% \caption{The left figure illustrates how accuracy disparity arises by minimizing the global square loss. The right figure gives a schematic illustration of the proposed algorithmic framework.}
% %   \vspace{-4mm}
% \end{figure*}

\begin{figure}[tb]
\centering
  \includegraphics[width=0.98\linewidth]{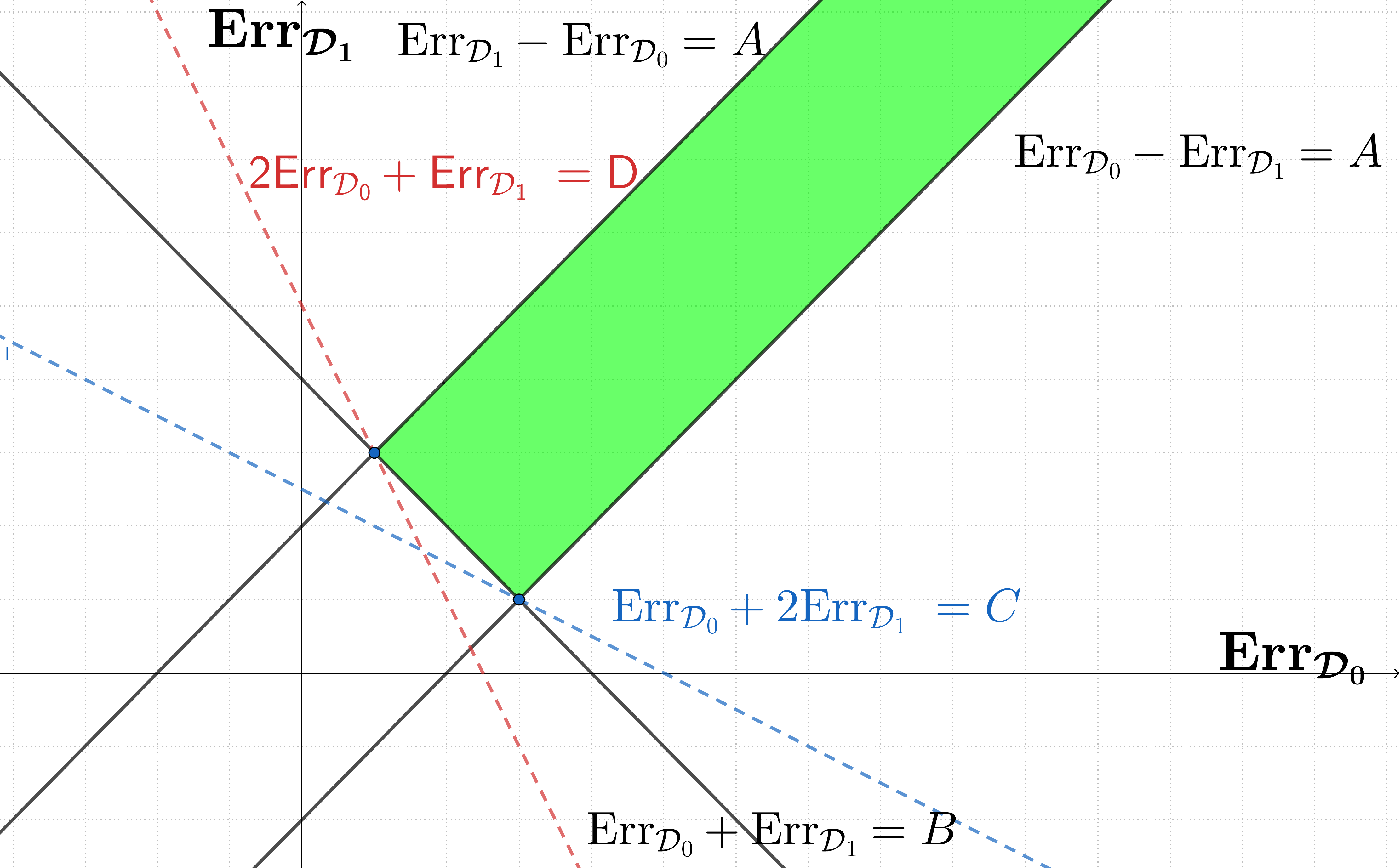}
  \caption{Geometric interpretation of accuracy disparity in regression.
  The green area corresponds to the feasible region of $\err_{\dist_0}$ and $\err_{\dist_1}$ under the hypothesis class $\HH$. For any optimal hypothesis $h$ which is solely designed to minimize the overall error, the best the hypothesis $h$ can do is to intersect with one of the two bottom vertices of the green area, leading to accuracy disparity if the width of the feasible region is nonzero. See section~\ref{geo-int} for more details.
  }
  \label{fig:geo}
\end{figure}

\section{Preliminaries}
\label{sec:prelim}
% We first introduce the notation used throughout the paper and describe the problem setup. We then briefly introduce the integral probability metrics~\citep{muller1997integral} that will be used in our analysis and algorithm design.
% We first introduce the notation used throughout the paper and describe the problem setup. We then briefly introduce the distance metrics that will be used in our analysis and algorithm design.

\paragraph{Notation}
We use $\xxspace\subseteq\RR^d$ and $\yyspace\subseteq\mathbb{R}$ to denote the input and output space. We use $X$ and $Y$ to denote random variables which take values in $\xxspace$ and $\yyspace$, respectively. Lower case letters $\xx$ and $y$ denote the instantiation of $X$ and $Y$. We use $H(X)$ to denote the Shannon entropy of random variable $X$, $H(X \mid Y)$ to denote the conditional entropy of $X$ given $Y$, and $I(X; Y)$ to denote the mutual information between $X$ and $Y$. To simplify the presentation, we use $A\in\{0, 1\}$ as the sensitive attribute, \eg, gender, race, etc. Let $\HH$ be the hypothesis class of regression models. In other words, for $h\in\HH$, $h:\xxspace\to\yyspace$ is a predictor. Note that even if the predictor does not explicitly take the sensitive attribute $A$ as an input variable, the prediction can still be biased due to the correlations with other input variables. In this work we study the stochastic setting where there is a joint distribution $\dist$ over $X, Y$ and $A$ from which the data are sampled. For $a\in\{0, 1\}$ and $y\in \mathbb{R}$, we use $\dist_a$ to denote the conditional distribution of $\dist$ given $A = a$ and $\dist^y$ to denote the conditional distribution of $\dist$ given $Y = y$. For an event $E$, $\dist(E)$ denotes the probability of $E$ under $\dist$. Given a feature transformation function $g:\xxspace\to\zzspace$ that maps instances from the input space $\xxspace$ to feature space $\zzspace$, we define $g_\sharp\dist\defeq \dist\circ g^{-1}$ to be the induced (pushforward) distribution of $\dist$ under $g$, \ie, for any event $E'\subseteq\zzspace$, $g_\sharp\dist(E') \defeq \dist(\{x\in\xxspace\mid g(x)\in E'\})$. We define $(\cdot)_+$ to be $\max\{\cdot, 0\}$. 

For regression problems, given a joint distribution $\dist$, the error of a predictor $h$ under $\dist$ is defined as $\err_\dist(h)\defeq\Exp_\dist[(Y - h(X))^2]$. 
To make the notation more compact, we may drop the subscript $\dist$ when it is clear from the context. Furthermore, we also use $\mse{\dist}{\Ypred}{Y}$ to denote the mean squared loss between the predicted variable $\Ypred = h(X)$ and the true label $Y$ over the joint distribution $\dist$. Similarly, we also use  $\crossentropy{\dist}{A}{\Apred}$ to denote the cross-entropy loss between the predicted variable $\Apred$ and the true label $A$ over the joint distribution $\dist$. Throughout the paper, we make the following standard boundedness assumption:
\begin{assumption}
There exists $M > 0$, such that for any hypothesis $\HH \ni h: \xxspace \to \yyspace$, $\|h\|_{\infty}\leq M$ and $|Y|\leq M$.
\label{ass:bound}
\end{assumption}

\paragraph{Problem Setup}
Our goal is to learn a regression model that is fair in the sense that the errors of the regressor are approximately equal across the groups given by the sensitive attribute $A$. We assume that the sensitive attribute $A$ is only available to the learner during the training phase and is not visible during the inference phase. We would like to point out that there are many other different and important definitions of fairness~\citep{narayanan2018translation} even in the sub-category of group fairness, and our discussion is by no means comprehensive. For example, two frequently used definitions of fairness in the literature are the so-called statistical parity~\citep{dwork2012fairness} and equalized odds~\citep{hardt2016equality}. Nevertheless, throughout this paper we mainly focus accuracy parity as our fairness notion, due to the fact that machine learning systems have been shown to exhibit substantial accuracy disparities between different demographic subgroups~\citep{barocas2016big, kim2016data, buolamwini2018gender}. This observation has already brought huge public attention (\eg, see New York Times, The Verge, and Insurance Journal) and calls for machine learning systems that (at least approximately) satisfy accuracy parity. For example, in a healthcare spending prediction system, stakeholders do not want the prediction error gaps to be too large among different demographic subgroups. Formally, accuracy parity is defined as follows:
\begin{definition} 
Given a joint distribution $\dist$, a predictor $h$ satisfies accuracy parity if $\err_{\dist_0}(h) = \err_{\dist_1}(h)$.
\end{definition}
In practice the exact equality of accuracy between two groups is often hard to ensure, so we define \emph{error gap} to measure how well the model satisfies accuracy parity:
\begin{definition}
Given a joint distribution $\dist$, the error gap of a hypothesis $h$ is $\errgap(h) \defeq |\err_{\dist_0}(h) - \err_{\dist_1}(h)|$.
\end{definition}

By definition, if a model satisfies accuracy parity, $\errgap(h)$ will be zero. 
Next we introduce two distance metrics that will be used in our theoretical analysis and algorithm design:
\begin{itemize}[leftmargin=*]
\item Total variation distance: it measures the largest possible difference between the probabilities that the two probability distributions can assign to the same event $E$. We use $\dtv(\distp, \distq)$ to denote the total variation:
        \begin{equation*}
        \dtv(\distp, \distq)\defeq \sup_{E}|\distp(E) - \distq(E)|.    
        \end{equation*}

\item Wasserstein distance: the Wasserstein distance between two probability distributions is
    \begin{equation}
    \nonumber
    W_1(\distp,\distq) = \sup_{f\in\{f: \| f \|_L \leq 1 \}} \left|\int_{\Omega} f d\distp - \int_{\Omega} f d\distq\right|,
    \end{equation} 
where $\| f \|_L$ is the Lipschitz semi-norm of a real-valued function of $f$ and $\Omega$ is the sample space over which two probability distributions $\distp$ and $\distq$ are defined. By the Kantorovich-Rubinstein duality theorem~\citep{villani2008optimal}, we recover the primal form of the Wasserstein distance, defined as 
    \begin{equation*}
    W_1(\distp, \distq) \defeq \inf_{\gamma \in \Gamma(\distp, \distq)} \int d(X, Y) \dif\gamma(X, Y),    
    \end{equation*}
where $\Gamma(\distp, \distq)$ denotes the collection of all couplings of $\distp$ and $\distq$, and $X$ and $Y$ denote the random variables with law $\distp$ and $\distq$ respectively. 
Throughout this paper we use $L_1$ distance for $d(\cdot, \cdot)$, but extensions to other distances, \eg, $L_2$ distance, is straightforward.
\end{itemize}   

\section{Main Results}
In this section, we first characterize why accuracy disparity arises in regression models. More specifically, given a hypothesis $h\in \hhspace$, we first prove a lower bound of joint errors. Then, we provide an error decomposition theorem which upper bounds the accuracy disparity and decompose it into the distance between marginal label distributions and the distance between conditional representations.
Based on these results, we give a geometric interpretation to visualize the feasible region of $\err_{\dist_0}$ and $\err_{\dist_1}$ and illustrate how error gap arises when learning a hypothesis $h$ that minimizes the global square error. 
Motivated by the error decomposition theorem, we propose two algorithms to reduce accuracy disparity, connect the game-theoretic optima of the objective functions in our algorithms with our theorems, and describe the practical implementations of the algorithms. Due to the space limit, we defer all the detailed proofs to the appendix.

\subsection{Bounds on Conditional Errors and Accuracy Disparity Gap}

Before we provide the prescriptive analysis of the accuracy disparity problem in regression, it is natural to ask whether accuracy parity is achievable in the first place. Hence, we first provide a sufficient condition to achieve accuracy parity in regression.

\begin{restatable}{proposition}{AccParSuff}
Assume both $\Exp_{\dist_a}[Y]$ and $\Exp_{\dist_a}[Y^2]$ are equivalent for any $A=a$, then using a constant predictor ensures accuracy parity in regression.
\label{prop:acc-parity}
\end{restatable}

Proposition~\ref{prop:acc-parity} states if the first two order moments of marginal label distributions are equal across different groups, then using a constant predictor leads to accuracy parity in regression.
Proposition~\ref{prop:acc-parity} is a relaxation of our proposed error decomposition theorem (Theorem~\ref{theorem:upper-bound-x-on-y}) which requires the total variation distance between group-wise marginal label distributions to be zero. 
However, the condition rarely holds in real-world scenarios and it does not provide any insights to algorithm design. 
Next we provide more in-depth analysis to understand why accuracy disparity appears in regression models and provide algorithm interventions to mitigate the problem. 

When we learn a predictor, the prediction function induces $X \overset{h}{\longrightarrow} \Ypred$, where $\Ypred$ is the predicted target variable given by hypothesis $h$. Hence for any distribution $\dist_0$ ($\dist_1$) of $X$, the predictor also induces a distribution $h_\sharp\dist_0$ ($h_\sharp\dist_1$) of $\Ypred$. Recall that the Wasserstein distance is metric, hence the following chain of triangle inequalities holds:
\begin{equation*}
    \small
    \centering
    \begin{aligned}
        W_1(\dist_0(Y), \dist_1(Y)) \leq&~  W_1(\dist_0(Y), h_\sharp\dist_0) + W_1(h_\sharp\dist_0, h_\sharp\dist_1) \\
        &~+ W_1(h_\sharp\dist_1, \dist_1(Y))
    \end{aligned}
    \label{equ:dist_triangle}
\end{equation*}

Intuitively, $W_1(\dist_a(Y), h_\sharp\dist_a)$ measures the distance between the true marginal label distribution and the predicted one when $A=a$. 
This distance is related to the prediction error of function $h$ conditioned on $A=a$:

\begin{restatable}{lemma}{ConditionalError}
Let $\Ypred= h(X)$, then for $a\in\{0, 1\}$, $W_1(\dist_a(Y), h_\sharp\dist_a) \leq \sqrt{\err_{\dist_a} (h)}$.
\label{lemma:w-dist}
\end{restatable}
Now we can get the following theorem that characterizes the lower bound of joint error on different groups:

\begin{restatable}{theorem}{ErrorLowerBound}
Let $\Ypred= h(X)$ be the predicted variable, then $\err_{\dist_0}(h) + \err_{\dist_1}(h) \geq \frac{1}{2}\big[\big(W_1(\dist_0(Y), \dist_1(Y)) - W_1(h_\sharp\dist_0, h_\sharp\dist_1)\big)_{+}\big]^2$.
 \label{theorem:lower-bound}
\end{restatable}

In Theorem~\ref{theorem:lower-bound}, we see that if the difference between marginal label distributions across groups is large, then statistical parity could potentially lead to a large joint error.
Moreover, Theorem~\ref{theorem:lower-bound} could be extended to give a lower bound on the joint error incurred by $h$ as well:
\begin{restatable}{corollary}{WeightedErrorLowerBound}
 Let $\Ypred = h(X)$ and $\alpha = \dist(A=0) \in [0, 1]$, we have $\err_{\dist}(h) \geq \frac{1}{2} \min\{\alpha, 1-\alpha\}\cdot\big[\big(W_1(\dist_0(Y), \dist_1(Y)) - W_1(h_\sharp\dist_0, h_\sharp\dist_1)\big)_{+}\big]^2$.
 \label{corollary:lower-bound}
\end{restatable}
Now we upper bound the error gap. We first relate the error gap to marginal label distributions and the predicted distributions conditioned on $Y=y$:

\begin{restatable}{theorem}{UpperXonY}
\label{theorem:upper-bound-x-on-y}
If Assumption~\ref{ass:bound} holds, then for $\forall h\in\HH$, let $\Ypred= h(X)$, the following inequality holds: 
\begin{equation*}
    \begin{aligned}
        \errgap(h) \leq&~ 8M^2 \dtv(\dist_0(Y), \dist_1(Y)) \\
        &~+ 3M \min\{\Exp_{\dist_0}[|\Exp_{\dist_0^y}[\Ypred]-\Exp_{\dist_1^y}[\Ypred]|],\,\\
        &~~\quad\quad\quad\quad\quad\Exp_{\dist_1}[|\Exp_{\dist_0^y}[\Ypred]-\Exp_{\dist_1^y}[\Ypred]|]\}.
    \end{aligned}
\end{equation*}
\end{restatable}

\paragraph{Remark}
We see that the error gap is upper bounded by two terms: the distance between marginal label distributions and the discrepancy between conditional predicted distributions across groups.
Given a dataset, the distance between marginal label distributions is a constant since the marginal label distributions are fixed. 
For the second term, if we can minimize the discrepancy of the conditional predicted distribution across groups, we then have a model that is free of accuracy disparity when the marginal label distributions are well aligned.

\paragraph{Geometric Interpretation}
\label{geo-int}
By Theorem~\ref{theorem:lower-bound} and Theorem~\ref{theorem:upper-bound-x-on-y}, we can visually illustrate how accuracy disparity arises given data distribution and the learned hypothesis that aims to minimize the global square error. In Figure~\ref{fig:geo}, given the hypothesis class $\HH$, we use the line $\err_{\dist_0} + \err_{\dist_1} = B$ to denote the lower bound in Theorem~\ref{theorem:lower-bound} and the two lines $|\err_{\dist_0} - \err_{\dist_1}| = A$ to denote the upper bound in Theorem~\ref{theorem:upper-bound-x-on-y}. These three lines form a feasible region (the green area) of $\err_{\dist_0}$ and $\err_{\dist_1}$ under the hypothesis class $\HH$. For any optimal hypothesis $h$ which is solely designed to minimize the overall error, the best the hypothesis $h$ can do is to intersect with one of the two bottom vertices. For example, the hypotheses (the red dotted line and the blue dotted line) trying to minimize overall error intersect with the two vertices of the region to achieve the smallest $\err_{\dist_0}$-intercept ($\err_{\dist_1}$-intercept), due to the imbalance between these two groups. However, since these two vertices are not on the diagonal of the feasible region, there is no guarantee that the hypothesis can satisfy accuracy parity ($\err_{\dist_0}=\err_{\dist_1}$), unless we can shrink the width of green area to zero.

\subsection{Algorithm Design}
Inspired by Theorem~\ref{theorem:upper-bound-x-on-y}, we can mitigate the error gap by aligning the group distributions via minimizing the distance of the conditional distributions across groups. However, it is intractable to do so explicitly in regression problems since $Y$ can take infinite values on $\mathbb{R}$. Next we will present two algorithms to approximately solve the problem through adversarial representation learning.

% by designing games whose equilibrium solutions correspond to transforms that perfectly align all the conditional distributions simultaneously.

Given a Markov chain $X \overset{g}{\longrightarrow} Z \overset{h}{\longrightarrow} \Ypred$, we are interested in learning group-invariant conditional representations so that the discrepancy between the induced conditional distributions $\dist_0^Y (Z=g(X))$ and $\dist_1^Y (Z=g(X))$ is minimized. In this case, the second term of the upper bound in Theorem~\ref{theorem:upper-bound-x-on-y} is minimized. However, it is in general not feasible since $Y$ is a continuous random variable. Instead, we propose to learn the representations of $Z$ to minimize the discrepancy between the joint distributions $\dist_0 (Z=g(X), Y)$ and $\dist_1 (Z=g(X), Y)$. Next, we will show the distances between conditional predicted distributions $\dist_0^Y (Z=g(X))$ and $\dist_1^Y (Z=g(X))$ are minimized when we minimize the joint distributions $\dist_0 (Z=g(X), Y)$ and $\dist_1 (Z=g(X), Y)$ in Theorem~\ref{theorem:optresponse} and Theorem~\ref{theorem:limitwass}.

% Instead of aligning the family of conditional distributions $\dist_0^y (Z=g(X))$ and $\dist_1^y (Z=g(X))$ for each $y\in\RR$ explicitly, we propose to learn the representations of $Z$ so that the discrepancy between the joint distributions $\dist_0 (Z=g(X), Y)$ and $\dist_1 (Z=g(X), Y)$ is minimized as much as possible. Note that it is in general not feasible to perfectly align $\dist_0 (Z=g(X), Y)$ and $\dist_1 (Z=g(X), Y)$ by solely learning the transform $Z$, due to the difference in the marginal distributions of $Y$ between two groups. Nevertheless, by proper choices of distance metrics between distributions, the optimal alignment with minimum discrepancy happens when the conditional feature distributions $\dist^Y_a(Z = g(X))$ are invariant w.r.t.\ the value of $a\in \{0, 1\}$.

To proceed, we first consider using the total variation distance to measure the distance between two distributions. In particular, we can choose to learn a binary discriminator $f:Z \times Y \longrightarrow \Apred$ that achieves minimum binary classification error on discriminating between points sampled from two distributions. In practice, we use the cross-entropy loss as a convex surrogate loss.
Formally, we are going to consider the following minimax game between $g$ and $f$:
\begin{equation}
    \min_{f\in \mathcal{F}}\max_{g}\quad \crossentropy{\dist}{A}{f(g(X), Y)}
\label{equ:game}
\end{equation}

Interestingly, for the above equation, the optimal feature transformation $g$ corresponds to the one that induces invariant conditional feature distributions.

% \begin{restatable}{proposition}{limitxent}
% For any feature map $g:\xxspace\to\zzspace$, assume that $\mathcal{F}$ contains all the randomized binary classifiers and $\mathcal{F} \ni f: \zzspace \times \yyspace \to \aaspace$, then  $\min_{f\in\mathcal{F}}\crossentropy{\dist}{A}{f(g(X), Y)} = H(A\mid Z,Y)$. 
% \label{theorem:limitxent}
% \end{restatable}
\begin{restatable}{theorem}{optresponse}
Consider the minimax game in~\eqref{equ:game}. The equilibrium $(g^*, f^*)$ of the game is attained when 1). $Z = g^*(X)$ is independent of $A$ conditioned on $Y$; 2). $f^*(Z,Y) = \dist(A = 1\mid Y, Z)$. 
\label{theorem:optresponse}
\end{restatable}
Since in the equilibrium of the game $Z$ is independent of $A$ conditioned on $Y$, the optimal $f^*(Z, Y)$ could also be equivalently written as $f^*(Z, Y) = \dist(A = 1\mid Y)$, \ie, the only useful information for the discriminator in the equilibrium is through the external information $Y$. In Theorem~\ref{theorem:optresponse}, the minimum cross-entropy loss that the discriminator (the equilibrium of the game) can achieve is $H(A\mid Z,Y)$ (see Proposition~\ref{prop:limitxent} in Appendix~\ref{app:proof}). For any feature transform $g$, by the basic property of conditional entropy, we have:
\begin{equation}
    \nonumber
    \begin{aligned}
        \min_{f\in\mathcal{F}}\crossentropy{\dist}{A}{f(g(X), Y)} &= H(A\mid Z,Y) \\
        &= H(A\mid Y) - I(A;Z \mid Y). %= H(A\mid Y) + I(A; Y) - I(A;Z,Y)
    \end{aligned}
\end{equation}

We know that $H(A\mid Y)$ is a constant given the data distribution. The maximization of $g$ in~\eqref{equ:game} is equivalent to the minimization of $\min_{Z  = g(X)}~I(A;Z \mid Y)$, and it follows that the optimal strategy for the transformation $g$ is the one that induces conditionally invariant features, \eg, $I(A;Z \mid Y) = 0$. 
% If $g^*$ plays optimally, then the optimal response of the discriminator $f$ is given by $f^*(Z,Y) = \dist(A = 1\mid Z = g^*(X), Y) = \dist(A = 1\mid Y)$. Put it in other words, the optimal strategy for the feature transformation function is the one so that the optimal response from the discriminator is to predict the group membership $A$ solely from the information of $Y$. 
Formally, we arrive at the following minimax problem:
%Recall that our goal is to align the distributions between $\dist_0^Y(Z)$ and $\dist_1^Y(Z)$, and it can be achieved by minimizing $I(A;Z \mid Y)$ (maximizing $\min_{f\in\mathcal{F}}\crossentropy{\dist}{A}{f}$ equivalently). Overall, it yields the following minimax optimization problem:
\begin{equation*}
\small
    \min_{h,g} \max_{f\in \mathcal{F}}~\mse{\dist}{h(g(X))}{Y} - \lambda\cdot\crossentropy{\dist}{A}{f(g(X), Y)}
\end{equation*}

In the above formulation, the first term corresponds to the minimization of prediction loss of the target task and the second term is the loss incurred by the adversary $f$. As a whole, the minimax optimization problem expresses a trade-off (controlled by the hyper-parameter $\lambda >0$) between accuracy and accuracy disparity through the representation learning function $g$. 

\paragraph{Wasserstein Variant} Similarly, if we choose to align joint distributions via minimizing Wasserstein distance, the following theorem holds.

\begin{restatable}{theorem}{limitwass}
Let the optimal feature transformation $g^* \defeq \argmin_{g} W_1(\dist_0(g(X),Y), \dist_1(g(X),Y))$, then $\dist_0^Y(Z=g^*(X))=\dist_1^Y(Z=g^*(X))$ almost surely.
\label{theorem:limitwass}
\end{restatable}

One notable advantage of using the Wasserstein distance instead of the TV distance is that, the Wasserstein distance is a continuous functional of both the feature map $g$ as well as the discriminator $f$~\citep{arjovsky2017wasserstein}. Furthermore, if both $g$ and $f$ are continuous functions of their corresponding model parameters, which is the case for models we are going to use in experiments, the objective function will be continuous in both model parameters. This property of the Wasserstein distance makes it more favorable from an optimization perspective. Using the dual formulation, equivalently, we can learn a Lipschitz function $f:Z \times Y \to \mathbb{R}$ as a witness function: 
\begin{equation*}
    \begin{aligned}
    \min_{h,g,Z_0\sim g_\sharp\dist_0, Z_1\sim g_\sharp\dist_1} \max_{f: \|f\|_L \leq 1}&\mse{\dist}{h(g(X))}{Y} \\
    &+ \lambda\cdot \big| f(Z_0, Y) - f(Z_1, Y) \big|.
    \end{aligned}
\end{equation*}

\paragraph{Game-Theoretic Interpretation}
We provide a game-theoretic interpretation of our algorithms in Figure~\ref{fig:game} to make our algorithms easier to follow. 

\begin{figure}[tb]
\centering
  \centering
  \includegraphics[width=\linewidth]{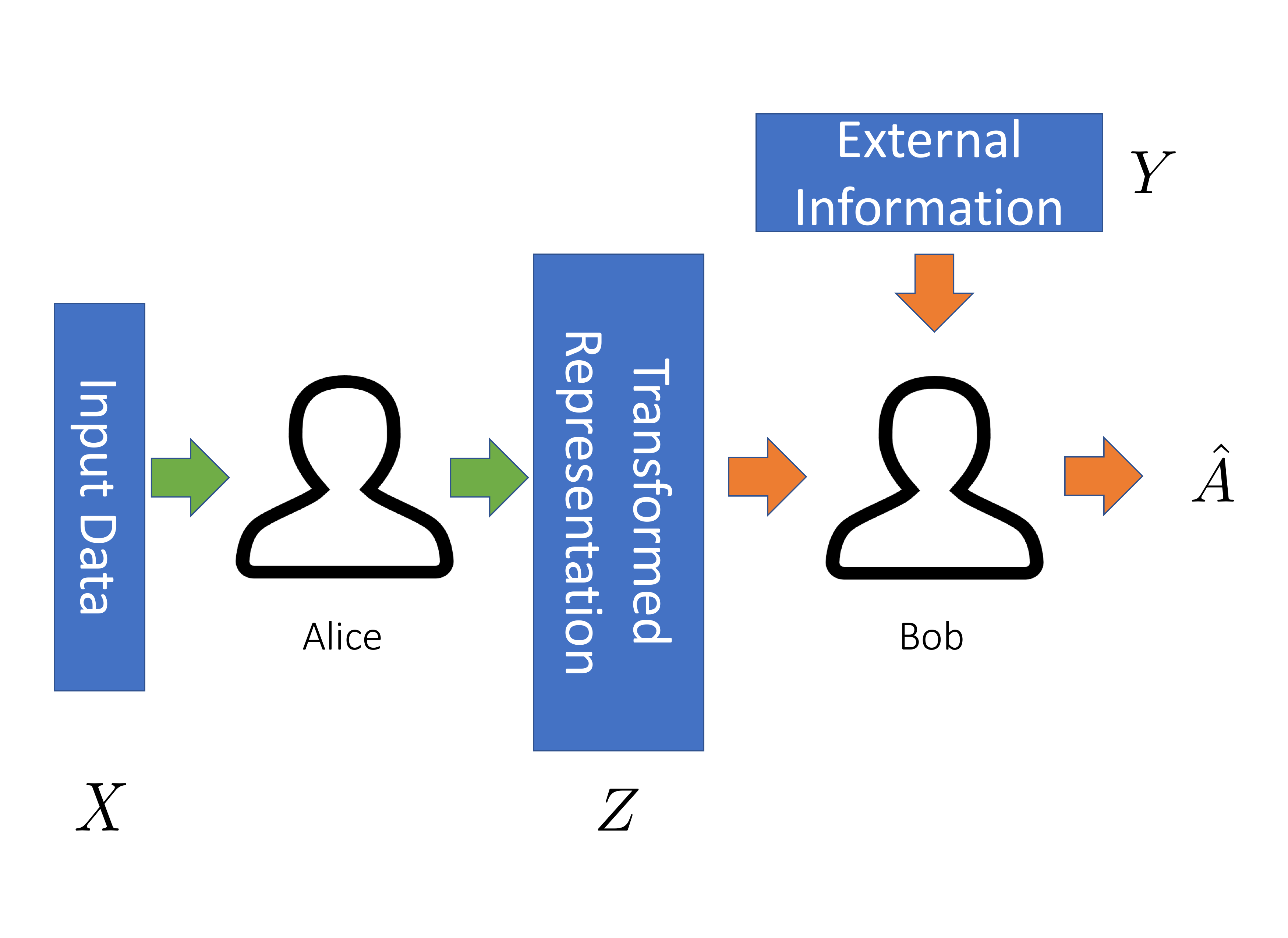}
\caption{The game-theoretic illustration of our algorithms. Bob's goal is to guess the group membership $A$ of each feature $Z$ sent by Alice with the corresponding labels $Y$ as the external information, while Alice's goal is to find a transformation from $X$ to $Z$ to confuse Bob.
}
\label{fig:game}
\end{figure}

As illustrated in Figure~\ref{fig:game}, consider Alice (encoder) and Bob (discriminator) participate a two-player game: upon receiving a set of inputs $X$, Alice applies a transformation to the inputs to generate the corresponding features $Z$ and then sends them to Bob. Besides the features sent by Alice, Bob also has access to the external information $Y$, which corresponds to the corresponding labels for the set of features sent by Alice. Once having both the features $Z$ and the corresponding labels $Y$ from external resources, Bob's goal is to guess the group membership $A$ of each feature sent by Alice, and to maximize his correctness as much as possible. On the other hand, Alice's goal is to compete with Bob, \ie, to find a transformation to confuse Bob as much as she can. Different from the traditional game without external information, here due to the external information $Y$ Bob has access to, Alice cannot hope to fully fool Bob, since Bob can gain some insights about the group membership $A$ of features from the external label information anyway. Nevertheless, Theorem~\ref{theorem:optresponse} and Theorem~\ref{theorem:limitwass} both state that when Bob uses a binary discriminator or a Wasstertein discriminator to learn $A$, the best Alice could do is to to learn a transformation $g$ so that the transformed representation $Z$ is insensitive to the values of A conditioned on any values of $Y=y$.

\begin{figure*}[!ht]
\centering
\begin{subfigure}[b]{.32\linewidth}
  \centering
  \includegraphics[width=\linewidth]{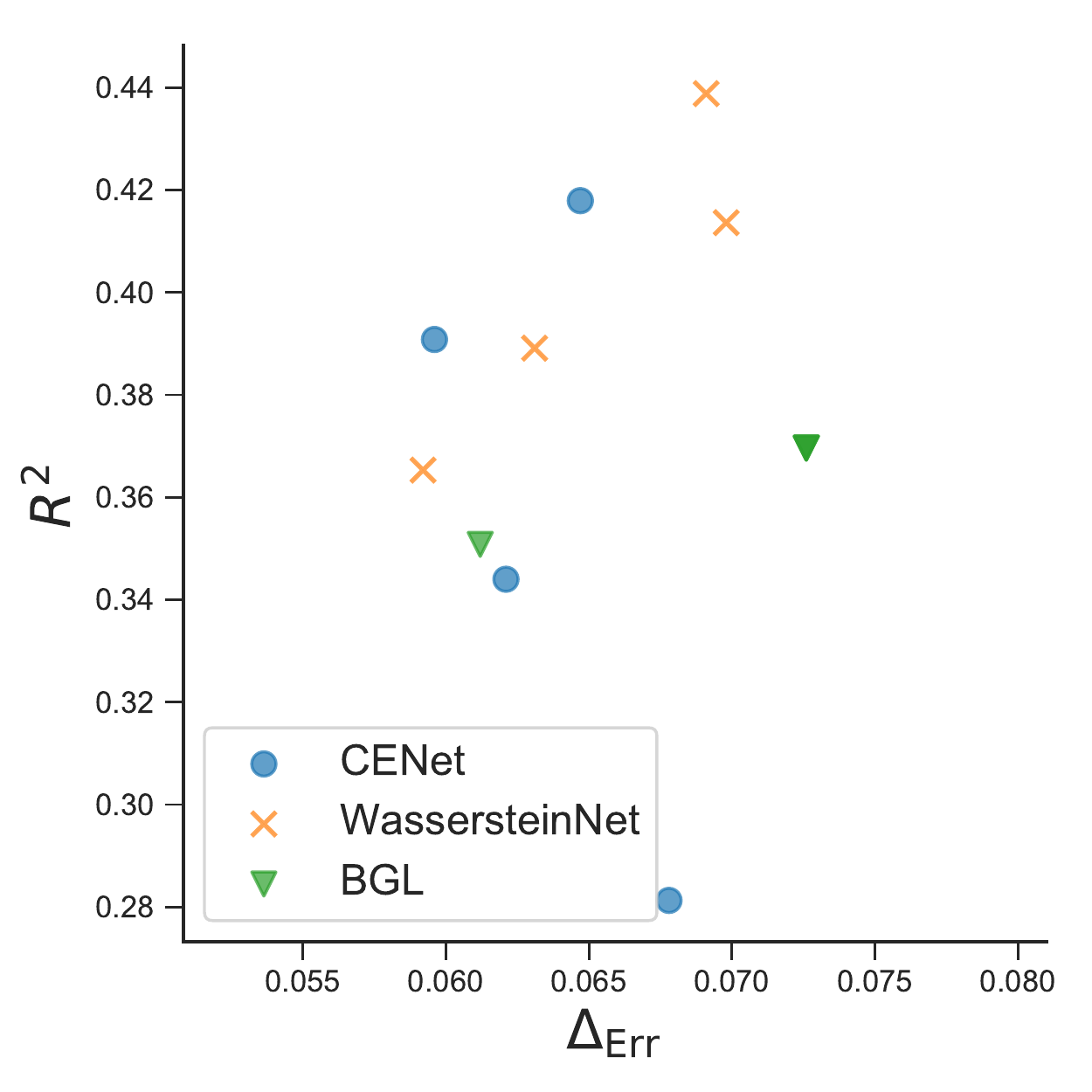}
    \caption{Adult}
  \label{fig:result-adult}
\end{subfigure}
~
\begin{subfigure}[b]{.32\linewidth}
  \centering
  \includegraphics[width=\linewidth]{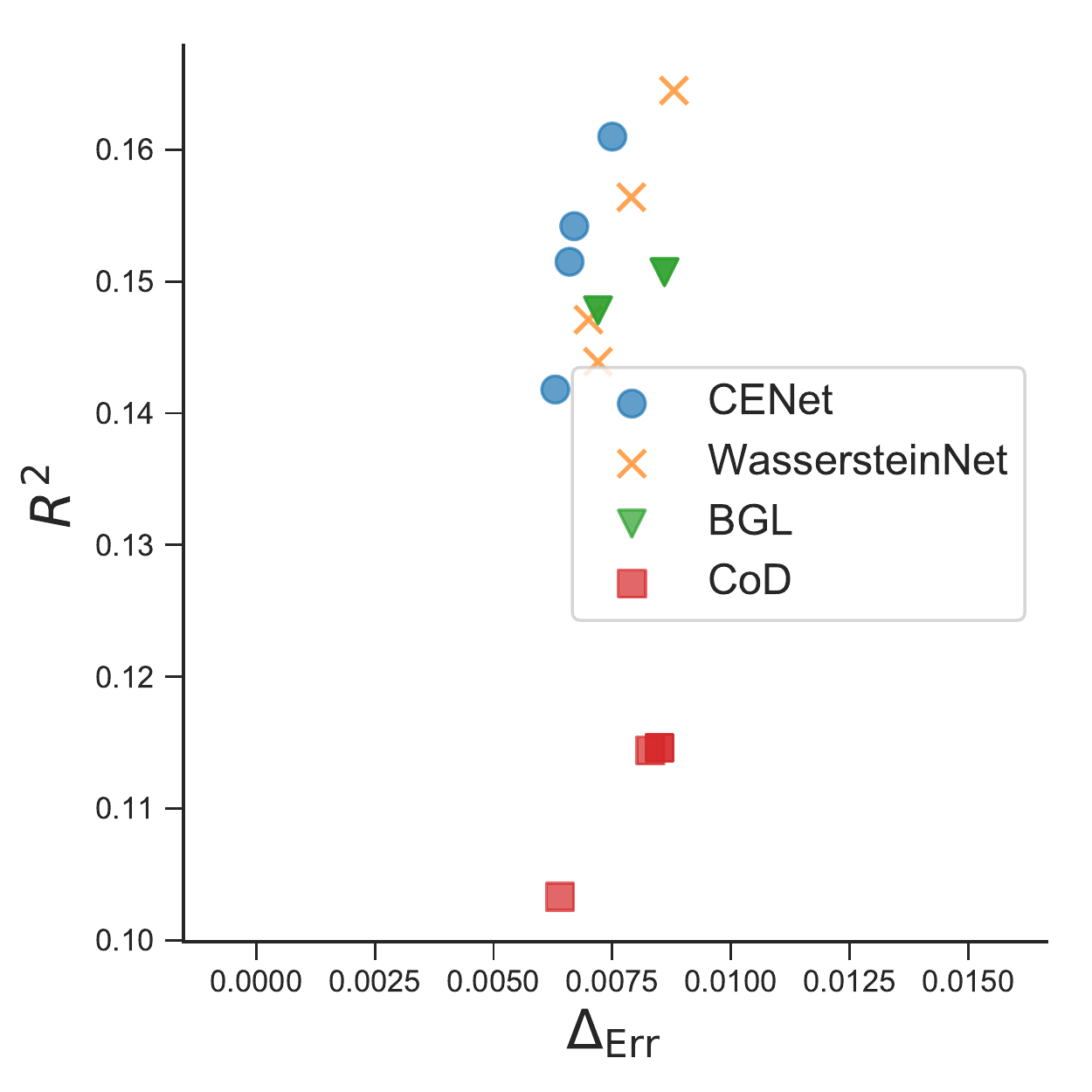}
  \caption{COMPAS}
  \label{fig:result-compas}
\end{subfigure}
~
\begin{subfigure}[b]{.32\linewidth}
  \centering
  \includegraphics[width=\linewidth]{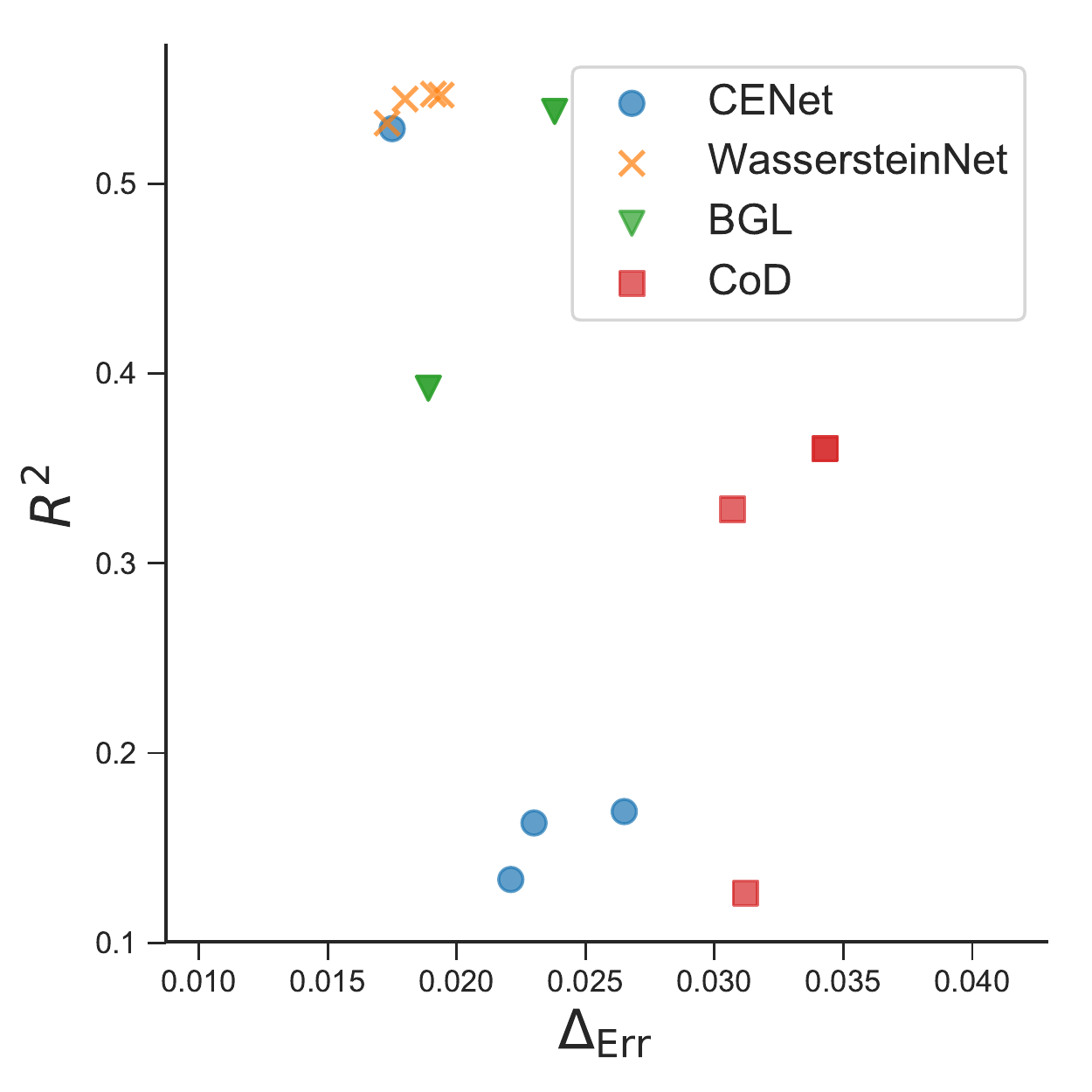}
  \caption{Crime}
  \label{fig:result-crime}
\end{subfigure}
~
\begin{subfigure}[b]{.32\linewidth}
  \centering
  \includegraphics[width=\linewidth]{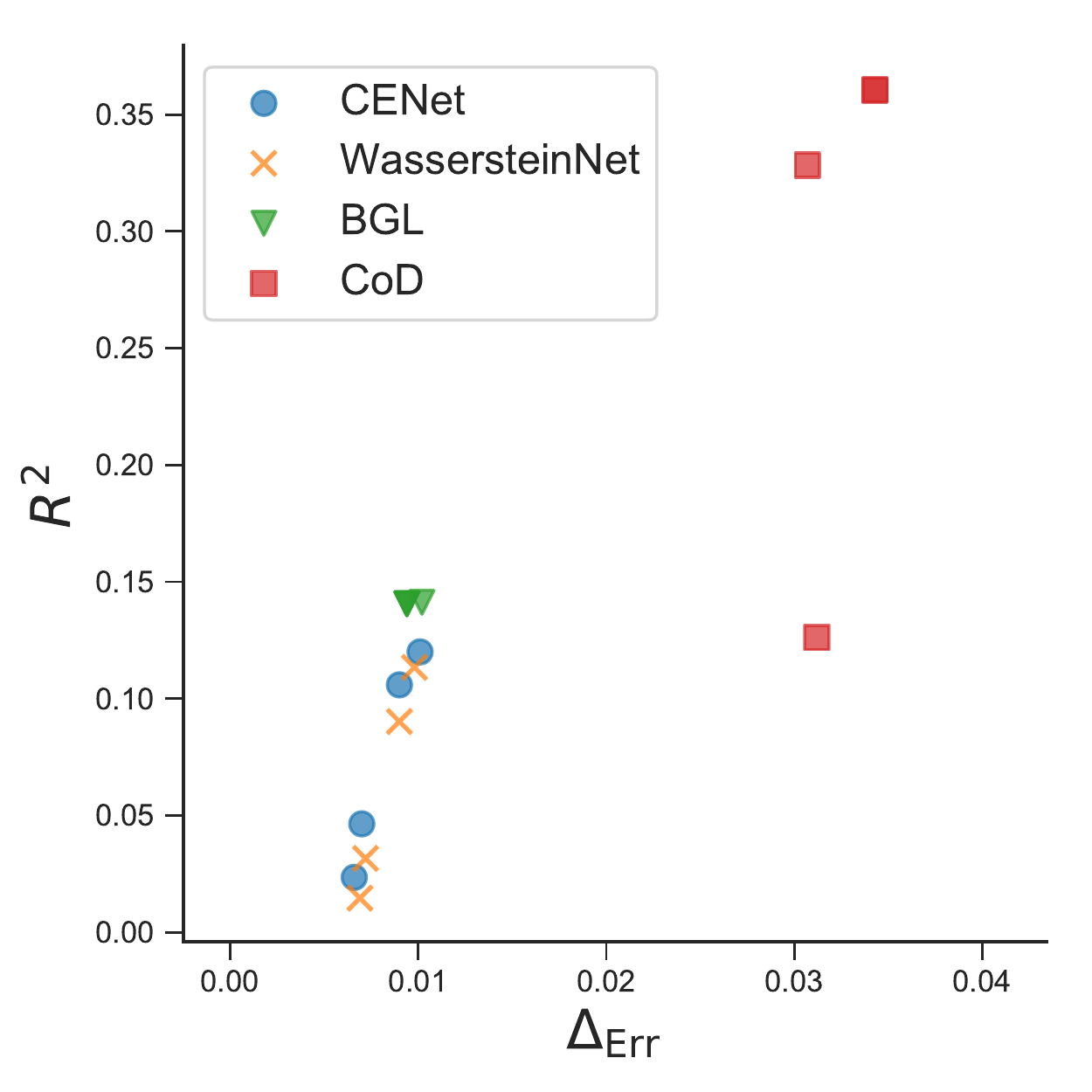}
  \caption{Law}
  \label{fig:result-law}
\end{subfigure}
~
\begin{subfigure}[b]{.32\linewidth}
  \centering
  \includegraphics[width=\linewidth]{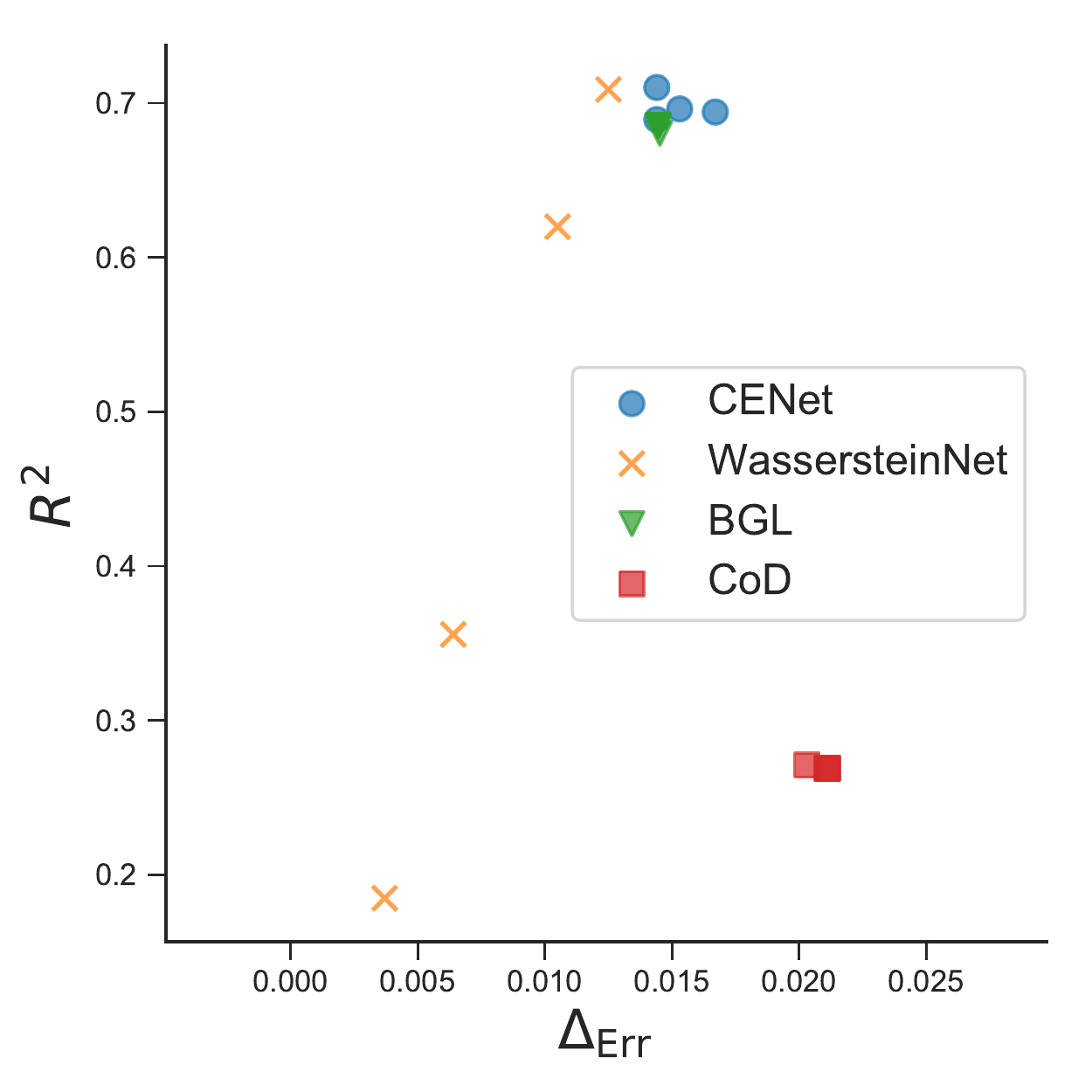}
  \caption{Insurance}
  \label{fig:result-insurance}
\end{subfigure}
\caption{Overall results: $R^2$ regression scores and error gaps of different methods in five datasets. Our goal is to achieve high $R^2$ scores with small error gap values (\ie, the most desirable points are located in the upper-left corner). 
% Our proposed methods are most effective in reducing the error gap values in all datasets compared to the baselines. 
Our proposed methods achieve the best trade-offs in Adult, COMPAS, Crime and Insurance datasets.}
%   \vspace*{-1em}
  \label{fig:overall-results}
\end{figure*}

\begin{figure*}[!ht]
\centering
\begin{subfigure}[b]{.32\linewidth}
  \centering
  \includegraphics[width=\linewidth]{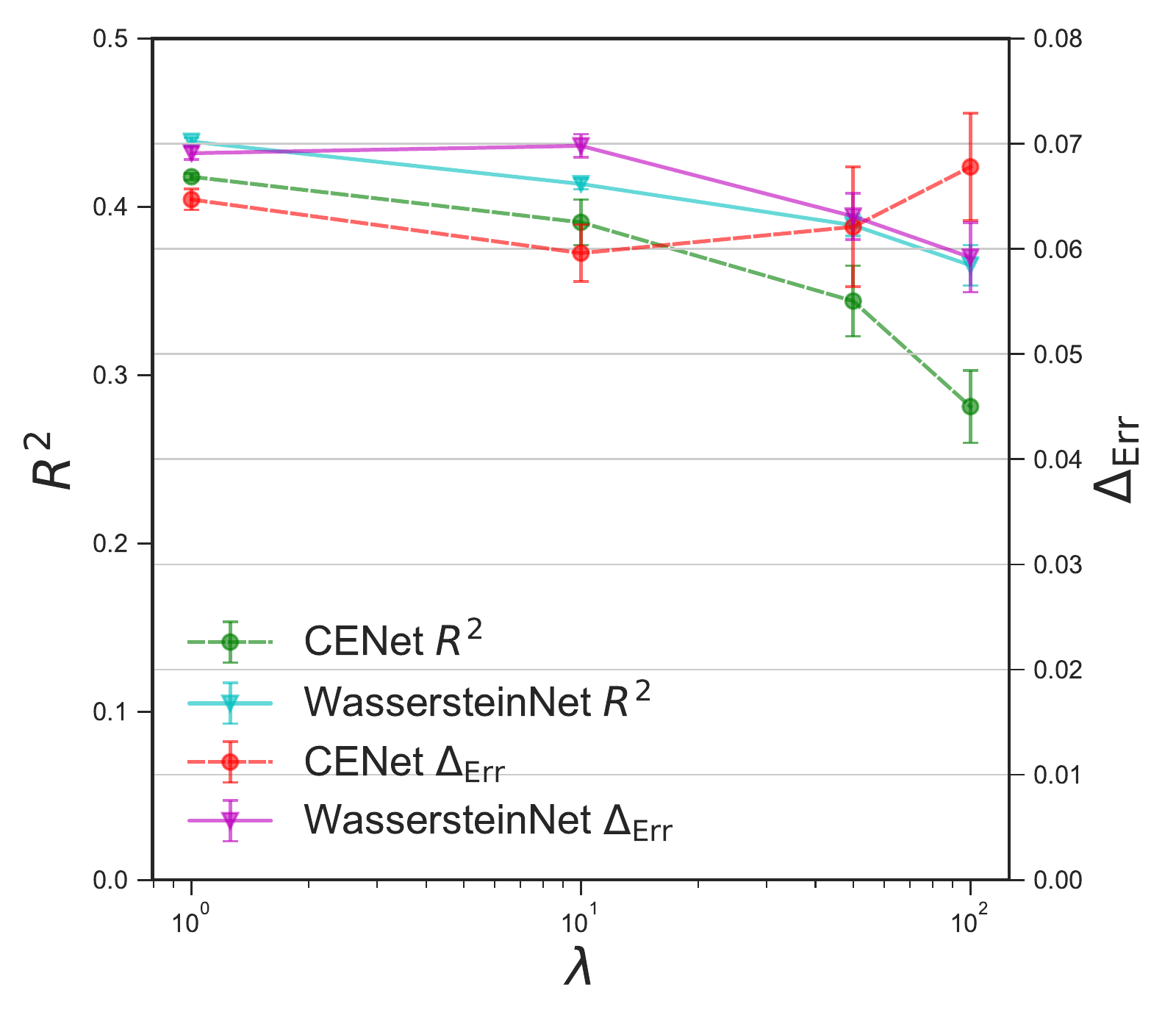}
    \caption{Adult}
  \label{fig:lambda-adult}
\end{subfigure}
~
\begin{subfigure}[b]{.32\linewidth}
  \centering
  \includegraphics[width=\linewidth]{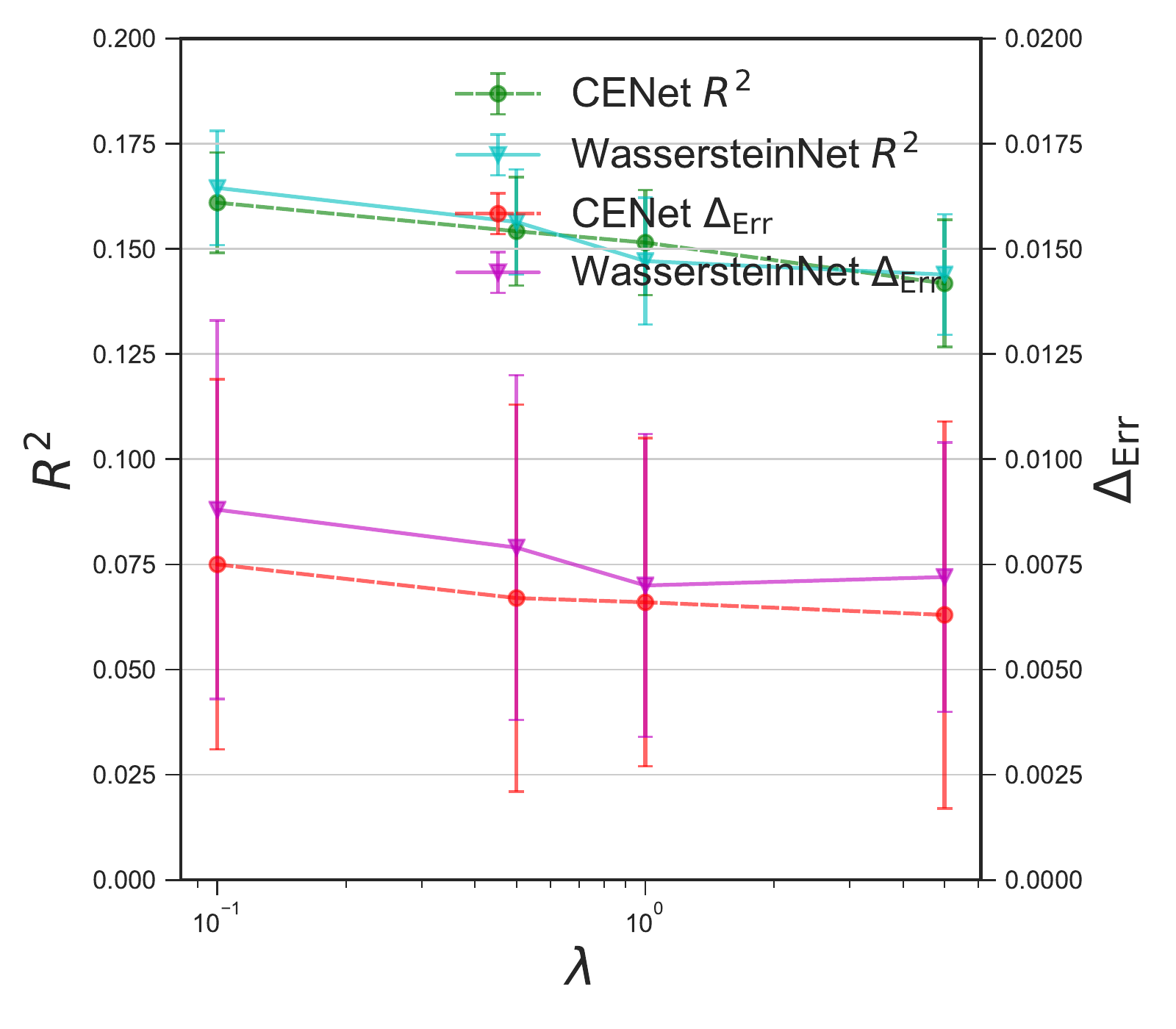}
  \caption{COMPAS}
  \label{fig:lambda-compas}
\end{subfigure}
~
\begin{subfigure}[b]{.32\linewidth}
  \centering
  \includegraphics[width=\linewidth]{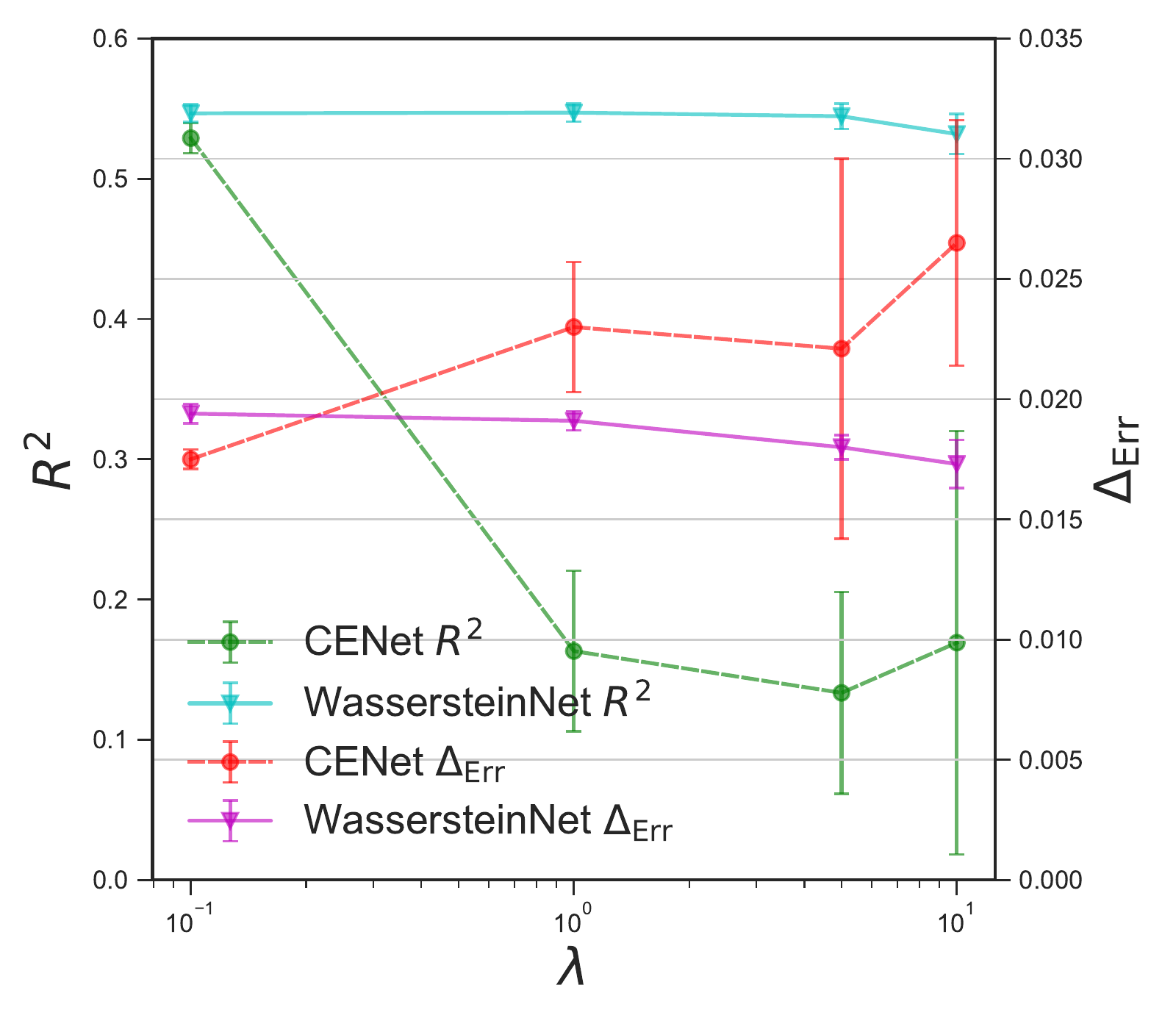}
  \caption{Crime}
  \label{fig:lambda-crime}
\end{subfigure}
~
\begin{subfigure}[b]{.32\linewidth}
  \centering
  \includegraphics[width=\linewidth]{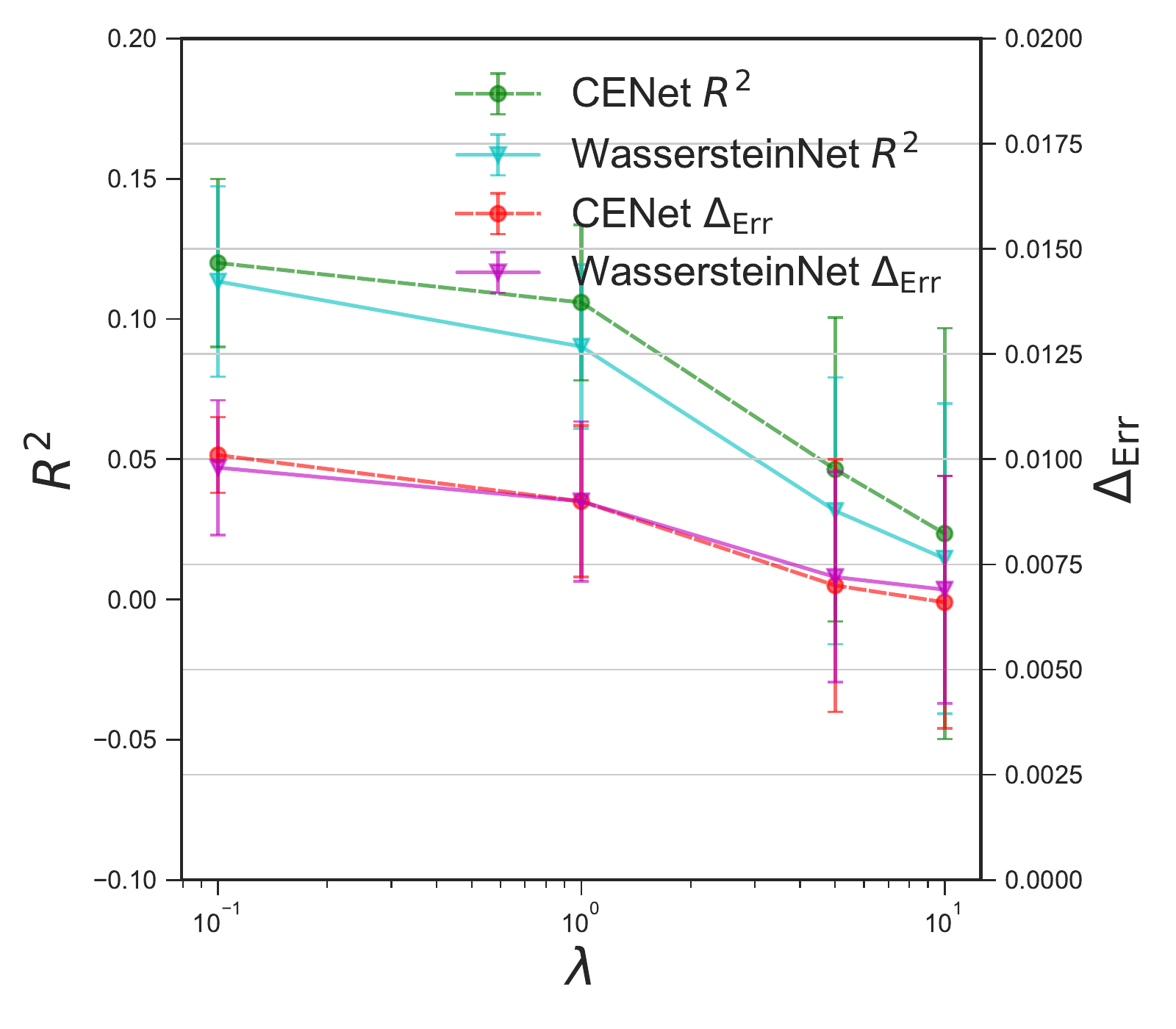}
  \caption{Law}
  \label{fig:lambda-law}
\end{subfigure}
~
\begin{subfigure}[b]{.32\linewidth}
  \centering
  \includegraphics[width=\linewidth]{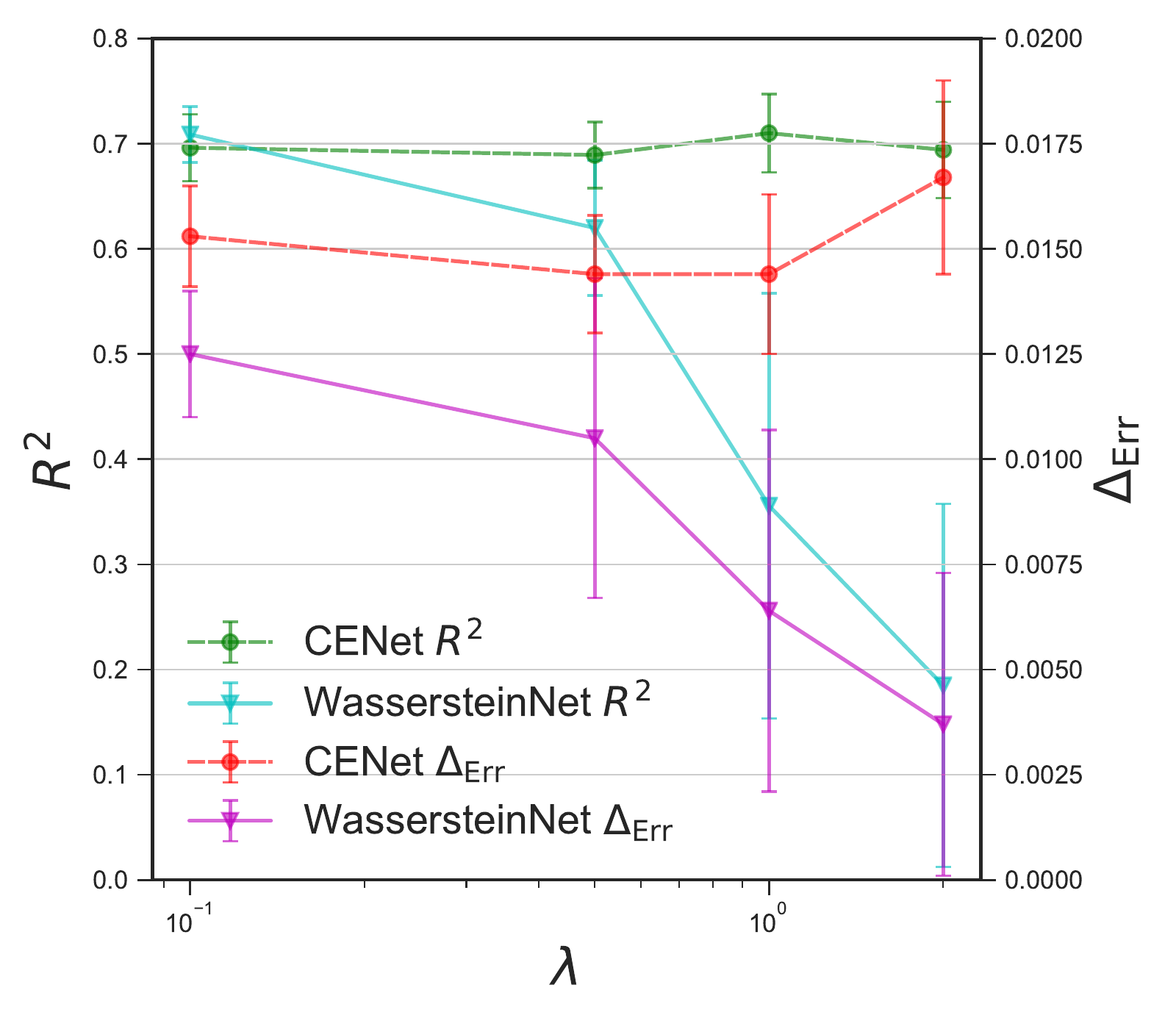}
  \caption{Insurance}
  \label{fig:lambda-insurance}
\end{subfigure}
\caption{$R^2$ regression scores and error gaps when $\lambda$ changes in \textsc{CENet} and \textsc{WassersteinNet}.
The general trend is that with the increase of $\lambda$, the error gap values and $R^2$ scores gradually decrease, except the cases where $\lambda$ increases in \textsc{CENet} in Adult, Crime and Insurance dataset. The exceptions are caused by the instability of the training processes of \textsc{CENet}~\citep{arjovsky2017towards}.
}
%   \vspace*{-1em}
  \label{fig:lambda-results}
\end{figure*}

\section{Experiments}
Inspired by our theoretical results that decompose accuracy disparity into the distance between marginal label distributions and the distance between conditional representations, we propose two algorithms to mitigate it. In this section, we conduct experiments to evaluate the effectiveness of our proposed algorithms in reducing the accuracy disparity.

\subsection{Experimental Setup}

\paragraph{Datasets}
We conduct experiments on five benchmark datasets: the Adult dataset~\citep{Dua:2019}, COMPAS dataset~\citep{dieterich2016compas}, Communities and Crime dataset~\citep{Dua:2019}, Law School dataset~\citep{wightman1998lsac} and Medical Insurance Cost dataset~\citep{lantz2013machine}. All datasets contain binary sensitive attributes (\eg, male/female, white/non-white). We refer readers to Appendix~\ref{app:exp-details} for detailed descriptions of the datasets and the data pre-processing pipelines. 
% Note that although the Adult dataset and COMPAS dataset are for binary classification tasks, we can still take them as regression tasks with two distinctive ordinal values.
Note that although the Adult and COMPAS datasets are for binary classification tasks, recent evidences~\citep{pmlr-v51-que16, muthukumar2020classification, hui2021evaluation} suggest that square loss achieves comparable performance with cross-entropy loss and hinge loss. In this regard, we take them as regression tasks with two distinctive ordinal values.

\paragraph{Methods}
We term the proposed algorithms \textsc{CENet} and \textsc{WassersteinNet} for our two proposed algorithms respectively and implement them using Pytorch~\citep{paszke2019pytorch}.\footnote{Our code is publicly available at:\\\url{https://github.com/JFChi/Understanding-and-Mitigating-Accuracy-Disparity-in-Regression}} To the best of our knowledge, no previous study aims to minimize accuracy disparity in regression using representation learning. However, there are other similar fairness notions and mitigation techniques proposed for regression and we add them as our baselines: (1) Bounded group loss (\textsc{BGL})~\citep{agarwal2019fair}, which asks for the prediction errors for any groups to remain below a predefined level $\epsilon$; (2) Coefficient of determination (\textsc{CoD})~\citep{komiyama2018nonconvex}, which asks for the coefficient of determination between the sensitive attributes and the predictions to remain below a predefined level $\epsilon$. 

For each dataset, we perform controlled experiments by fixing the regression model architectures to be the same. We train the regression models via minimizing mean squared loss. Among all methods, we vary the trade-off parameter (\ie, $\lambda$ in \textsc{CENet} and \textsc{WassersteinNet} and $\epsilon$ in \textsc{BGL} and \textsc{CoD}) and report and the corresponding $R^2$ scores and the error gap values.
% \footnote{We also report the corresponding classification accuracy for Adult and COMPAS datasets in Appendix~\ref{app:add-exp-res}.} 
For each experiment, we average the results for ten different random seeds.
Note that \textsc{CoD} cannot be implemented on the Adult dataset since the size of the Adult dataset is large and the QCQP optimization algorithm to solve \textsc{CoD} needs a quadratic memory usage of the dataset size.
We refer readers to Appendix~\ref{app:exp-details} for detailed hyper-parameter settings in our experiments and Appendix~\ref{app:add-exp-res} for additional experimental results.

\subsection{Results and Analysis}

The overall results are visualized in Figure~\ref{fig:overall-results}.
The following summarizes our observations and analyses:
(1) Our proposed methods \textsc{WassersteinNet} and \textsc{CENet} are most effective in reducing the error gap values in all datasets compared to the baselines. Our proposed methods also achieve the best trade-offs in Adult, COMPAS, Crime and Insurance datasets: with the similar error gap values ($R^2$ scores), our methods achieve the highest $R^2$ scores (lowest error gap values). In the Law dataset, the error gap values decrease with high utility losses in our proposed methods due the significant trade-offs between the predictive power of the regressors and accuracy parity. 
We suspect this is because the feature noise distribution in one group differs significantly than the others in the Law dataset.
(2) Among our proposed methods, \textsc{WassersteinNet} are more effective in reducing the error gap values while \textsc{CENet} might fail to decrease the error gaps in Adult, Crime and Insurance datasets and might even cause non-negligible reductions in the predictive performance of the regressors in Adult and Crime datasets. 
The reason behind it is that the minimax optimization in the training of \textsc{CENet} could lead to an unstable training process under the presence of a noisy approximation to the optimal discriminator~\citep{arjovsky2017towards}. We will provide more analysis in Figure~\ref{fig:lambda-results} next.
(3) Compared to our proposed methods, \textsc{BGL} and \textsc{CoD} can also decrease error gaps to a certain extent. This is because: (i) BGL aims to keep errors remaining relatively low in each group, which helps to reduce accuracy disparity; (ii) CoD aims to reduce the correlation between the sensitive attributes and the predictions (or the inputs) in the feature space, which might somehow reduce the dependency between the distributions of these two variables. 
% In comparison, our proposed methods do better in mitigating accuracy disparity.

We further analyze how the trade-off parameter $\lambda$ in the objective functions affect the performance of our methods.
Figure~\ref{fig:lambda-results} shows $R^2$ regression scores and error gaps when $\lambda$ changes in \textsc{CENet} and \textsc{WassersteinNet}. 
% We see that the error gap gradually decreases with the increase of the trade-off parameter $\lambda$ in most scenarios with small accuracy loss (except for \textsc{CENet} in Adult dataset and Crime dataset when $\lambda$ is large), which demonstrates the overall effectiveness of our proposed algorithms. 
We see the general trend is that with the increase of the trade-off parameter $\lambda$, the error gap values and $R^2$ scores gradually decrease.
Plus, the increase of $\lambda$ generally leads to the instability of training processes with larger variances of both $R^2$ scores and error gap values. 
In Adult, Crime and Insurance datasets,  \textsc{WassersteinNet} is more effective in mitigating accuracy disparity when $\lambda$ increases, while \textsc{CENet} fails to decrease the error gap values and might suffer from significant accuracy loss. 
% It is not surprising since the estimation of total variation in minimax optimization could lead to an unstable training process~\citep{arjovsky2017towards}.
The failure to decrease the error gap values with  significant accuracy loss and variance indicates the estimation of total variation in minimax optimization for \textsc{CENet} could lead to a highly unstable training process~\citep{arjovsky2017towards}.

\section{Related Work}

\paragraph{Algorithmic Fairness}
In the literature, two main notions of fairness, i.e., \emph{group fairness} and \emph{individual fairness}, has been widely studied~\citep{dwork2012fairness, zemel2013learning,feldman2015certifying,zafar2017fairness-a,hardt2016equality,zafar2017fairness-b, hashimoto2018fairness, madras2019fairness}.  
% In particular, \citet{feldman2015certifying} proposed to repair data distribution to remove disparate impact by using finding the median distribution of input feature with respect to the sensitive attribute.
% In particular, \citet{zafar2017fairness-a} proposed to use the covariance between sensitive attributes and the signed distance from the input features to the decision boundary of the classifier as constraints in convex optimization to reduce disparate mistreatment for binary classification;
%\citet{hashimoto2018fairness} found that accuracy disparity can be amplified over time with feedback signals; 
In particular, \citet{chen2018my} analyzed the impact of data collection on discrimination (\eg, false positive rate, false negative rate, and zero-one loss) from the perspectives of bias-variance-noise decomposition, and they suggested collecting more training examples and collect additional variables to reduce discrimination. 
\citet{khani2019noise} argued that the loss difference among different groups is determined by the amount of latent (unobservable) feature noise and the difference between means, variances, and sizes of the groups with an assumption that there are a latent random feature and a noise feature that are involved in the generation of the observable features.
\citet{khani2020removing} further found out that spurious features from inputs can hurt accuracy and affect groups disproportionately.
\citet{zhao2019inherent} proposed an error decomposition theorem which upper bounds accuracy disparity in the classification setting by three terms: the sum of group-wise noise, the distance of marginal input distributions across groups and the discrepancy of group-wise optimal decision functions. However, their error decomposition theorem does not lead to any mitigation approaches in classification: minimizing the distance of marginal input distributions across groups does not necessarily mitigate accuracy disparity since it could possibly exacerbate the noise term and the discrepancy of group-wise optimal decision functions in the meantime. Besides, the optimal group-wise decision functions are unknown and intractable to approximate in the feature spaces, which also adds to the difficulty of applying their upper bound directly.
% However, their error decomposition theorem does not lead to algorithm interventions to mitigate accuracy disparity in classification settings.
In comparison, our work only assumes that there is a joint distribution where all variables are sampled and precisely characterizes disparate predictive accuracy in regression in terms of the distance between marginal label distributions and the distance between conditional representations. Inspired by our theoretical results, we also propose practical algorithms to mitigate the problem when collecting more data becomes infeasible.

\paragraph{Fair Regression}
A series of works focus on fairness under the regression problems~\citep{calders2013controlling,johnson2016impartial, berk2018fairness, komiyama2018nonconvex, chzhen2020fair-b, bigot2020statistical}. 
% Specifically, \citet{berk2017convex} introduced a family of convex (group/individual) fairness regularizers for linear and logistics regressions; \citet{perez2017fair} studied kernel methods to reduce cross-covariance between the sensitive attributes and target variables and ensure statistical independence in linear regression;
%\citet{agarwal2019fair} studied regression under the notion of statistical parity and bounded group loss.
To the best of our knowledge, no previous study aimed to minimize accuracy disparity in regression from representation learning. However, there are different fairness notions and techniques proposed for regression: \citet{agarwal2019fair} proposed fair regression with bounded group loss (\ie, it asks that the prediction error for any protected group remains below some pre-defined level) and used exponentiated-gradient approach to satisfy BGL. \citet{komiyama2018nonconvex} aimed to reduce the coefficient of determination between the sensitive attributes between the predictions to some pre-defined level and used an off-the-shelf convex optimizer to solve the problem.
\citet{mary2019fairness} used the Hirschfeld-Gebelein-R\'enyi Maximum Correlation Coefficient to generalize fairness measurement to continuous variables and ensured equalized odds (demographic parity) constraint by minimizing the  $\chi^2$ divergence between the predicted variable and the sensitive variable (conditioned on target variable).
\citet{zink2020fair} considered regression problems in health care spending and proposed five fairness criteria (\eg, covariance constraint, net compensation penalization, etc.) in the healthcare domain.
\citet{narasimhan2020pairwise} proposed pairwise fairness notions (\eg, pairwise equal opportunity requires each pair from two arbitrary different groups to be equally-likely to be ranked correctly) for ranking and regression models.
\citet{chzhen2020fair-a} studied the regression problem with demographic parity constraint and showed the optimal fair predictor is achieved in the Wasserstein barycenter of group distributions.
In contrast, we source out the root of accuracy disparity in regression through the lens of information theory and reduce it via distributional alignment using TV distance and Wasserstein distance in the minimax games.

\paragraph{Fair Representation}
%Due to the power of learning rich representation of deep neural networks, 
A line of works focus on building algorithmic fair decision making systems using adversarial techniques to learn fair
representations~\citep{edwards2015censoring,beutel2017data,zhao2019conditional}. The main idea behind is to learn a good representation of the data so that the data owner can maximize the accuracy while removing the information related to the sensitive attribute. \citet{madras2018learning} proposed a generalized framework to learn adversarially fair and transferable representations and suggests using the label information in the adversary to learn equalized odds or equal opportunity representations in the classification setting. Apart from adversarial representation, recent work also proposed to use distance metrics, \eg, the maximum mean discrepancy~\citep{louizos2015variational} and the Wasserstein distance~\citep{jiang2019wasserstein} to remove group-related information. Prior to this work, it is not clear aligning conditional distributions via adversarial representation learning could lead to (approximate) accuracy parity. 
Our analysis is the first work to connect accuracy parity and (conditional) distributional alignment in regression and we also provide algorithm interventions to mitigate the problem where it is challenging to align conditional distributions in regression problems.
% Compared to their work, we propose to align (conditional) distributions across groups to reduce accuracy disparity using minimax optimization and analyze the game-theoretic optima in the minimax game in the regression setting.

\section{Conclusion}

In this paper, we theoretically and empirically study accuracy disparity in regression problems. Specifically, we prove an information-theoretic lower bound on the joint error and a complementary upper bound on the error gap across groups to depict the feasible region of group-wise errors. Our theoretical results indicate that accuracy disparity occurs inevitably due to the marginal label distributions differ across groups. To reduce such disparity, we further propose to achieve accuracy parity by learning conditional group-invariant representations using statistical distances. 
The game-theoretic optima of the objective functions in our proposed methods are achieved when the accuracy disparity is minimized. 
Our empirical results on five benchmark datasets demonstrate that our proposed algorithms help to reduce accuracy disparity effectively. We believe our results take an important step towards better understanding accuracy disparity in machine learning models.

% \section*{Software and Data}

\section*{Acknowledgements}

We thank anonymous reviewers for their insightful feedback and suggestions. JC and YT would like to acknowledge support from NSF CNS 1823325, NSF CNS 1850479, and NSF OAC 2002985. HZ thanks the DARPA XAI project, contract \#FA87501720152, for support. GG thanks Microsoft Research for support.


\bibliography{reference}
\bibliographystyle{plainnat}


\newpage
\onecolumn
\appendix
\section*{Appendix}
\label{sec:appendix}
In the appendix, we give the proofs of the theorems and claims in our paper, the experimental details and more experimental results.

\section{Missing Proofs}
\label{app:proof}

\AccParSuff*

\begin{proof}
For $a \in \{0, 1\}$, we have
\begin{equation}
    \nonumber
    \begin{aligned}
        &~\err_{\dist_a}(h) \\
        =&~ \Exp_{\dist_a}[(h(X)-Y)^2]\\ 
        =&~ \Exp_{\dist_a}[(h(X)-\Exp_{\dist_a}(Y)+\Exp_{\dist_a}(Y)-Y)^2]\\
        =&~ \Exp_{\dist_a}[(h(X)-\Exp_{\dist_a}(Y))^2] + \Exp_{\dist_a}[(Y-\Exp_{\dist_a}(Y))^2] - 2\,\Exp_{\dist_a}[(h(X)-\Exp_{\dist_a}(Y))(Y-\Exp_{\dist_a}(Y))].\\
        % =&~ \Exp_{\dist_a}[(h(X)-\Exp_{\dist_a}[Y|X])^2] + \Exp_{\dist_a}[(Y-\Exp_{\dist_a}[Y|X])^2] \\
        % =&~ \Exp_{\dist_a}[\Var_{\dist_a}[Y|X]]   .\\
    \end{aligned}
\end{equation}
It is easy to see the first two terms are equal across different groups since $\Exp_{\dist_a}[Y]$, $\Exp_{\dist_a}[Y^2]$ and $h(X)$ are the same across different groups. For the third term, we have
\begin{equation}
    \nonumber
    \begin{aligned}
    &~\Exp_{\dist_a}[(h(X)-\Exp_{\dist_a}(Y))(Y-\Exp_{\dist_a}(Y))] \\
    =&~ \Exp_{\dist_a(X)}[\Exp_{\dist_a(Y\mid X)}[(h(X)-\Exp_{\dist_a}(Y))(Y-\Exp_{\dist_a}(Y))\mid X]] \\
    =&~  \Exp_{\dist_a(X)}[ (h(X)-\Exp_{\dist_a}[Y\mid X])( \Exp_{\dist_a}[Y\mid X]) -\Exp_{\dist_a}[Y\mid X]) ]\\
    =&~ 0.\\
    \end{aligned}
\end{equation}
Thus, the errors across different groups made by the constant predictor are the same if $\Exp_{\dist_a}[Y]$ and $\Exp_{\dist_a}[Y^2]$ are equivalent across different groups.
% Thus, $\Exp_{\dist_a}(Y)$ and $\Var(Y\mid A=a)$ are equivalent for any $A=a$, then learning a constant predictor ensures accuracy parity in regression.
\end{proof}

\ConditionalError*
\begin{proof}
 The prediction error conditioned on $a\in\{0, 1\}$ is
    \begin{equation}
    \nonumber
    \begin{aligned}
    \err_{\dist_a} (h) &=~ \Exp[\big(Y - h(X)\big)^2 | A=a] \\
    &\geq~ \Exp^2[|Y - h(X)| | A=a] \\
    & \geq~ \big( \inf_{\Gamma(\dist_a(Y), \dist_a(h(X)))} \Exp[|Y - h(X)|] \big)^2 \\
    &=~ W_1^2(\dist_a(Y), h_\sharp\dist_a).
    \end{aligned}
    \end{equation}
    Taking square root at both sides then
completes the proof.
\end{proof}

\ErrorLowerBound*
\begin{proof}
Since $W_1(\cdot, \cdot)$ is a distance metric, the result follows immediately the triangle inequality and Lemma~\ref{lemma:w-dist}:
\begin{equation}
    \nonumber
    \begin{aligned}
        W_1(\dist_0(Y), \dist_1(Y)) &\leq  \sqrt{\err_{\dist_0} (h)} + W_1(h_\sharp\dist_0, h_\sharp\dist_1) + \sqrt{\err_{\dist_1} (h)}.
    \end{aligned}
\end{equation}
Rearrange the equation above and by AM-GM inequality, we have
\begin{equation}
    \nonumber
    \begin{aligned}
        W_1(\dist_0(Y), \dist_1(Y)) - W_1(h_\sharp\dist_0, h_\sharp\dist_1)
        \leq \sqrt{\err_{\dist_0} (h)} + \sqrt{\err_{\dist_1} (h)}
        \leq \sqrt{2(\err_{\dist_0} (h)+ \err_{\dist_1} (h))}.
    \end{aligned}
\end{equation}
Taking square at both sides then completes the proof.
\end{proof}

\WeightedErrorLowerBound*

\begin{proof}
The joint error is
\begin{equation}
    \nonumber
    \begin{aligned}
        &~\err_{\dist}(h) \\
        =&~\alpha\,\err_{\dist_0}(h) + (1-\alpha)\, \err_{\dist_1}(h)\\
        \geq&~ \min\{\alpha, 1-\alpha\} \big(\err_{\dist_0}(h) + \err_{\dist_1}(h) \big)\\
        \geq&~ \frac{1}{2} \min\{\alpha, 1-\alpha\} [\big(W_1(\dist_0(Y), \dist_1(Y)) - W_1(h_\sharp\dist_0, h_\sharp\dist_1)\big)_+]^2. &&\text{(Theorem~\ref{theorem:lower-bound})}
    \end{aligned}
\end{equation}
\end{proof}

\UpperXonY*
\begin{proof}
First, we show that for $a \in \{0, 1\}$:
\begin{equation}
    \nonumber
    \begin{aligned}
    \err_{\dist_a}(h) = \Exp_{\dist_a}[(h(X)-Y)^2]= \Exp_{\dist_a}[h^2(X)-2Yh(X)+Y^2] = \Exp_{\dist_a}[h^2(X)-2Yh(X)] + \Exp_{\dist_a}[Y^2]. \\
    \end{aligned}
\end{equation}
Next, we bound the error gap:
\begin{equation}
    \nonumber
    \begin{aligned}
    &~| \err_{\dist_0}(h) - \err_{\dist_1}(h) | \\
    =&~ |\Exp_{\dist_0}[h^2(X)-2Yh(X)] + \Exp_{\dist_0}[Y^2] - \Exp_{\dist_1}[h^2(X)-2Yh(X)] - \Exp_{\dist_1}[Y^2] | \\
    \leq&~ |\Exp_{\dist_0}[h^2(X)-2Yh(X)] -\Exp_{\dist_1}[h^2(X)-2Yh(X)]| + |\Exp_{\dist_0}[Y^2] - \Exp_{\dist_1}[Y^2]|. && \text{(Triangle inequality)}\\
    \end{aligned}
\end{equation}

For the second term, we can easily prove that 
\begin{equation}
\nonumber
|\Exp_{\dist_0}[Y^2]-\Exp_{\dist_1}[Y^2]| = |\inp{Y^2}{d\dist_0 - d\dist_1}| \leq \| Y \|^2_{\infty} \|d\dist_0 - d\dist_1\|_1  \leq 2M^2 \dtv(\dist_0(Y), \dist_1(Y)),
\end{equation}
where the second equation follows Hölder's inequality and the last equation follow the definition of total variation distance. Now it suffices to bound the remaining term:
\begin{equation}
    \nonumber
    \begin{aligned}
        &~|\Exp_{\dist_0}[h^2(X)-2Yh(X)] -\Exp_{\dist_1}[h^2(X)-2Yh(X)]|\\
        =&~ \bigg| \int h(\xx) (h(\xx)-2y)\dif\mu_0(\xx, y) -  \int h(\xx) (h(\xx)-2y)\dif\mu_1(\xx, y)  \bigg|\\
        \leq&~ \bigg| \iint h(\xx) (h(\xx)-2y)\dif\mu_0(\xx | y)d\mu_0(y) - \iint h(\xx) (h(\xx)-2y)\dif\mu_0(\xx| y)d\mu_1(y) \bigg|  && \text{(Triangle inequality)}\\
        &~+ \bigg| \iint h(\xx) (h(\xx)-2y)\dif\mu_1(\xx | y)d\mu_1(y) - \iint h(\xx) (h(\xx)-2y)\dif\mu_0(\xx| y)d\mu_1(y)  \bigg|.
    \end{aligned}
\end{equation}

We upper bound the first term: 
\begin{equation}
    \nonumber
    \begin{aligned}
    &~ \bigg| \iint h(\xx) (h(\xx)-2y)\dif\mu_0(\xx | y)\dif\mu_0(y) - \iint h(\xx) (h(\xx)-2y)\dif\mu_0(\xx| y)\dif\mu_1(y) \bigg| \\
    \leq&~ \iint \big| h(\xx) (h(\xx)-2y) (\dif\mu_0(y) - \dif\mu_1(y))\big| \dif\mu_0(\xx| y)   \\
    \leq&~ \int \big|\dif\mu_0(y) - \dif\mu_1(y)\big| \int \big|\sup_{\xx}h(\xx)\big| \big|h(\xx)-2y\big| \dif\mu_0(\xx| y)\\
    \leq&~ M \int \Exp_{\dist_0}[|h(X)-2Y||Y=y] \,\big|\dif\mu_0(y) - \dif\mu_1(y)\big| && \text{(Assumption~\ref{ass:bound})} \\
    \leq&~ 3M^2 \int \big|\dif\mu_0(y) - \dif\mu_1(y)\big| && \text{(Assumption~\ref{ass:bound})} \\
    \leq&~ 6M^2 \dtv(\dist_0(Y), \dist_1(Y)). \\
    \end{aligned}
\end{equation}

Note that the last equation follows the definition of total variation distance. For the second term, we have:
\begin{equation}
    \nonumber
    \begin{aligned}
        &~\bigg| \iint h(\xx) (h(\xx)-2y)\dif\mu_1(\xx | y)\dif\mu_1(y) - \iint h(\xx) (h(\xx)-2y)\dif\mu_0(\xx| y)\dif\mu_1(y) \bigg| \\
        \leq&~ \bigg|\iint h^2(\xx)(\dif\mu_1(\xx|y)-\dif\mu_0(\xx|y))\dif\mu_1(y) \bigg| + \bigg|\iint 2y\,h(\xx)(\dif\mu_1(\xx|y)-\dif\mu_0(\xx|y))\dif\mu_1(y) \bigg| && \text{(Triangle inequality)}\\
        \leq&~ 3M~ \Exp_{\dist_1}[|\Exp_{\dist_0^y}[\Ypred]-\Exp_{\dist_1^y}[\Ypred]|]. && \text{(Assumption~\ref{ass:bound})}
    \end{aligned}
\end{equation}

To prove the last equation, we first see that:
\begin{equation}
    \nonumber
    \begin{aligned}
        &~ \bigg|\iint h^2(\xx)(\dif\mu_1(\xx|y)-\dif\mu_0(\xx|y))\dif\mu_1(y)\bigg| \\
        \leq&~  \bigg|\iint \big(\sup_{\xx}h(\xx)\big) h(\xx)(\dif\mu_1(\xx|y)-\dif\mu_0(\xx|y)) \dif \mu_1(y)\bigg| \\ 
        \leq&~ M \int \big|\Exp_{\dist_0}[h(X)|Y=y] - \Exp_{\dist_1}[h(X)|Y=y]\big| \dif\mu_1(y) && \text{(Assumption~\ref{ass:bound})} \\
        =&~ M~\Exp_{\dist_1}[|\Exp_{\dist_0^y}[\Ypred]-\Exp_{\dist_1^y}[\Ypred]|].
    \end{aligned}
\end{equation}

Similarly, we also have:
\begin{equation}
    \nonumber
    \begin{aligned}
    &~\bigg|\iint 2y\,h(\xx)(\dif\mu_1(\xx|y)-\dif\mu_0(\xx|y))\dif\mu_1(y)\bigg| \\
    \leq&~ 2~ \bigg| \iint (\sup y) h(\xx)(\dif\mu_1(\xx|y)-\dif\mu_0(\xx|y)) \dif\mu_1(y) \bigg| \\
    \leq&\,2M \int \big|\Exp_{\dist_0}[h(X)|Y=y] - \Exp_{\dist_1}[h(X)|Y=y]\big| \,d\mu_1(y) && \text{(Assumption~\ref{ass:bound})}\\
    =&~ 2 M~\Exp_{\dist_1}[|\Exp_{\dist_0^y}[\Ypred]-\Exp_{\dist_1^y}[\Ypred]|]. \\
    \end{aligned}
\end{equation}

By symmetry, we can also see that:
\begin{equation}
    \nonumber
    \begin{aligned}
        & |\Exp_{\dist_0}[h^2(X)-2Yh(X)] -\Exp_{\dist_1}[h^2(X)-2Yh(X)]|\leq 6M^2 \dtv(\dist_0(Y), \dist_1(Y)) + 3M\,\Exp_{\dist_1}[|\Exp_{\dist_0^y}[\Ypred]-\Exp_{\dist_1^y}[\Ypred]|]. \\
    \end{aligned}
\end{equation}

Combine the above two equations yielding:
\begin{equation}
    \nonumber
    \begin{aligned}
    &~|\Exp_{\dist_0}[h^2(X)-2Yh(X)] -\Exp_{\dist_1}[h^2(X)-2Yh(X)]| \\
    \leq&~ 6M^2 \dtv(\dist_0(Y), \dist_1(Y)) + 3M\min\{ \Exp_{\dist_0}[|\Exp_{\dist_0^y}[\Ypred]-\Exp_{\dist_1^y}[\Ypred]|], \Exp_{\dist_1}[|\Exp_{\dist_0^y}[\Ypred]-\Exp_{\dist_1^y}[\Ypred]|]\}. \\
    \end{aligned}
\end{equation}

Incorporating the terms back to the upper bound of the error gap then completes the proof.
\end{proof}

\optresponse*
\begin{proof}
To prove Theorem~\ref{theorem:optresponse}, we first give Proposition~\ref{prop:limitxent}.

\begin{restatable}{proposition}{limitxent}
For any feature map $g:\xxspace\to\zzspace$, assume that $\mathcal{F}$ contains all the randomized binary classifiers and $\mathcal{F} \ni f: \zzspace \times \yyspace \to \aaspace$, then  $\min_{f\in\mathcal{F}}\crossentropy{\dist}{A}{f(g(X), Y)} = H(A\mid Z,Y)$. 
\label{prop:limitxent}
\end{restatable}
% \limitxent*

\begin{proof}
    By the definition of cross-entropy loss, we have:
    \begin{equation}
        \nonumber
        \begin{aligned}
            \crossentropy{\dist}{A}{f} &= -\Exp_{\dist} \left[\ind(A = 0)\log(1-f(g(X),Y)) + \ind(A=1)\log(f(g(X),Y))\right] \\
            &= -\Exp_{g_\sharp\dist} \left[\ind(A = 0)\log(1-f(Z,Y)) + \ind(A=1)\log(f(Z,Y))\right] \\
            &= -\Exp_{Z,Y}\Exp_{A\mid Z,Y}\left[\ind(A = 0)\log(1-f(Z,Y)) + \ind(A=1)\log(f(Z,Y))\right] \\
            &= -\Exp_{Z,Y}\left[\dist(A = 0\mid Z,Y)\log(1-f(Z,Y)) + \dist(A=1\mid Z,Y)\log(f(Z,Y))\right] \\
            &= \Exp_{Z,Y}\left[\kl(\dist(A\mid Z,Y)~\|~ f(Z,Y))\right] + H(A\mid Z,Y) \\
            &\geq H(A\mid Z,Y), \\
        \end{aligned}
    \end{equation}
    where $\kl(\cdot\|\cdot)$ denotes the KL divergence between two distributions. From the above inequality, it is also clear that the minimum value of the cross-entropy loss is achieved when $f(Z,Y)$ equals the conditional probability $\dist(A=1\mid Z,Y)$, i.e., $ f^*(Z,Y) = \dist(A = 1\mid Z = g(X),Y)$.
\end{proof}

Proposition~\ref{prop:limitxent} states that the minimum cross-entropy loss that the discriminator can achieve is $H(A\mid Z,Y)$ when $f$ is the conditional distribution $\dist(A = 1\mid Z = g(X),Y)$. By the basic property of conditional entropy, we have:
\begin{equation}
    \nonumber
    \min_{f\in\mathcal{F}}\crossentropy{\dist}{A}{f(g(X), Y)} = H(A\mid Z,Y) = H(A\mid Y) - I(A;Z \mid Y). %= H(A\mid Y) + I(A; Y) - I(A;Z,Y)
\end{equation}

Note that $H(A\mid Y)$ is a constant given the distribution $\dist$, so the maximization of $g$ is equivalent to the minimization of $\min_{Z  = g(X)}~I(A;Z \mid Y)$, and it follows that the optimal strategy for the transformation $g$ is the one that induces conditionally invariant features, e.g., $I(A;Z \mid Y) = 0$. On the other hand, if $g^*$ plays optimally, then the optimal response of the discriminator $f$ is given by 
\begin{equation*}
    f^*(Z,Y) = \dist(A = 1\mid Z = g^*(X), Y) = \dist(A = 1\mid Y).
\end{equation*}
\end{proof}

\limitwass*
\begin{proof}
By the definition of Wasstertein distance, we have:
\begin{equation}
\nonumber
\begin{aligned}
    W_1(\dist_0(Z,Y), \dist_1(Z,Y)) &= \inf_{\gamma \in \Gamma(\dist_0, \dist_1)} \int d((\zz_0,y_0), (\zz_1,y_1)) \dif\gamma((\zz_0,y_0), (\zz_1,y_1)) \\
    &= \inf_{\gamma \in \Gamma(\dist_0, \dist_1)} \iint d((\zz_0,y_0), (\zz_1,y_1)) \dif\gamma(\zz_0, \zz_1\mid y_0, y_1) \dif\gamma(y_0, y_1) \\
    &=  \inf_{\gamma \in \Gamma(\dist_0, \dist_1)} \iint \|\zz_0 - \zz_1\|_1 +|y_0 - y_1| \dif\gamma(\zz_0, \zz_1\mid y_0, y_1) \dif\gamma(y_0, y_1) \\
    &\geq \inf_{\gamma \in \Gamma(\dist_0, \dist_1)} \iint |y_0 - y_1| \dif \gamma(y_0, y_1) \dif\gamma(\zz_0, \zz_1\mid y_0, y_1) \\
    &= \inf_{\gamma \in \Gamma(\dist_0(Y), \dist_1(Y))} \int |y_0 - y_1| \dif \gamma(y_0, y_1) \\
    &= W_1 (\dist_0(Y), \dist_1(Y)). \\
\end{aligned}
\end{equation}

To finish the proof, next we prove the lower bound is achieved when $\dist_0^Y(Z=g^*(X))=\dist_1^Y(Z=g^*(X))$: it is easy to see  $W_1(\dist_0^Y(Z), \dist_0^Y(Z)) = \int \|\zz_0 - \zz_1\|_1 \dif\gamma(\zz_0, \zz_1\mid y_0, y_1) = 0$ when the conditional distributions are equal. In this case, when the Wasserstein distance is minimized, then $Z$ is conditionally independent of $A$ given $Y$ almost surely.
% In this case, the equation can be achieved in the above inequality. 
\end{proof}

\section{Experimental Details}
\label{app:exp-details}

\paragraph{Adult}
The Adult dataset contains 48,842 examples for income prediction. The task is to predict whether the annual income of an individual is greater or less than 50K/year based on the attributes of the individual, such as education level, age, occupation, etc. In our experiment, we use gender (binary) as the sensitive attribute. The target variable (income) is an ordinal binary variable: 0 if $<$ 50K/year otherwise 1. After data pre-processing, the dataset contains 30,162/15,060 training/test instances where the input dimension of each instance is 113. 
We show the data distributions for different demographic subgroups in Table~\ref{tab:adult-dist}.

To preprocess the dataset, we first filter out the data records that contain the missing values. We then remove the sensitive attribute from the input features and normalize the input features with its means and standard deviations. Note that we use one-hot encoding for the categorical attributes. 
%We visualize the data distributions for different demographic subgroups in Figure~\ref{fig:data-dist-adult}.

For our proposed methods, we use a three-layer neural network with ReLU as the activation function of the hidden layers and the sigmoid function as the output function for the prediction task (we take the first two layers as the feature mapping). The number of neurons in the hidden layers is 60. We train the neural networks with the \textsc{Adadelta} algorithm with the learning rate 0.1 and a batch size of 512. The models are trained in 50 epochs. For the adversary networks in \textsc{CENet} and \textsc{WassersteinNet}, we use a two-layer neural network with ReLU as the activation function. The number of neurons in the hidden layers of the adversary networks is 60. The adversary network in \textsc{CENet} also uses sigmoid function as the output function. The weight clipping norm in the adversary network of \textsc{WassersteinNet} is 0.005. We use the gradient reversal layer \citep{ganin2016domain} to implement the gradient descent ascent (GDA) algorithm for optimization of the minimax problem since it makes the training process more stable \citep{daskalakis2018limit}. For the rest of the datasets we used in our experiments, we also use a gradient reversal layer to implement our algorithms.

We use the Fairlearn toolkit~\citep{bird2020fairlearn} to implement \textsc{BGL}: we use the exponentiated-gradient algorithm with the default setting as the mitigator and vary the upper bound $\epsilon\in\{0.1, 0.2, 0.3, 0.5\}$ of the bounded group loss constraint. For each value of $\epsilon$, we average the results of ten different random seeds.

\paragraph{COMPAS} The COMPAS dataset contains 6,172 instances to predict whether a criminal defendant will recidivate within two years or not. It contains attributes such as age, race, etc. In our experiment, we use race
(white or non-white) as the sensitive attribute and recidivism as the target variable. We split the dataset into train and test sets with the ratio 7/3. We show the data distributions for different demographic subgroups in Table~\ref{tab:compas-dist}.

For all methods, we use a two-layer neural network with ReLU as the activation function of the hidden layers and the sigmoid function as the output function for the prediction task (we take the first layer as the feature mapping). The number of neurons in the hidden layers is 60. We train the neural networks with the \textsc{Adadelta} algorithm with the learning rate 1.0 and a batch size of 512. The models are trained in 50 epochs. For the adversary networks in \textsc{CENet} and \textsc{WassersteinNet}, we use a two-layer neural network with ReLU as the activation function. The number of neurons in the hidden layers of the adversary networks is 10. The adversary network in \textsc{CENet} also uses sigmoid function as the output function. The weight clipping norm in the adversary network of \textsc{WassersteinNet} is 0.05. 

We use the Fairlearn toolkit to implement \textsc{BGL}: we use the exponentiated-gradient algorithm with the default setting as the mitigator and vary the upper bound $\epsilon\in\{0.1, 0.2, 0.3, 0.5\}$ of the bounded group loss constraint. For each value of $\epsilon$, we average the results of ten different random seeds.

As for \textsc{CoD}, we follow the source implementation.\footnote{https://github.com/jkomiyama/fairregresion} We use the same hyper-parameter settings as~\citep{komiyama2018nonconvex}: We use the kernelized optimization with the random Fourier features and the RBF kernel (we vary hyper-parameter of the RBF kernel $\gamma\in\{0.1, 1.0, 10, 100\}$) and report the best results with minimal MSE loss for each time we change the fairness budget $\epsilon$.  We also vary $\epsilon\in\{0.01, 0.1, 0.5, 1.0\}$ and average the results of ten different random seeds.

\begin{table}[htb]
    \centering
    \begin{minipage}[c]{.49\linewidth}
    \centering
    \captionsetup{justification=centering}
    \caption{Data distribution of $Y$ and\\ $A$ in Adult dataset.}
    \begin{tabular}{ccc}\toprule
         &  $Y = 0$ & $Y = 1$\\\midrule
    $A = 0$ & 20988 & 9539 \\
    $A = 1$ & 13026 & 1669 \\\bottomrule
    \label{tab:adult-dist}
    \end{tabular}
    \end{minipage}
    \begin{minipage}[c]{.49\linewidth}
    \centering
    \captionsetup{justification=centering}
    \caption{Data distribution of $Y$ and\\ $A$ in COMPAS dataset.}
    \begin{tabular}{ccc}\toprule
         &  $Y = 0$ & $Y = 1$\\\midrule
    $A = 0$ & 1849 & 1148 \\
    $A = 1$ & 1514 & 1661 \\\bottomrule
    \label{tab:compas-dist}
    \end{tabular}
    \end{minipage}
    
\end{table}

\paragraph{Communities and Crime}
The Communities and Crime dataset contains 1,994 examples of socio-economic, law enforcement, and crime data about communities in the United States. The task is to predict the number of violent crimes per 100K population. All attributes in the dataset have been curated and normalized to $[0, 1]$. In our experiment, we use race (binary) as the sensitive attribute: 1 if the population percentage of the white is greater or equal to 80\% otherwise 0. After data pre-processing, the dataset contains 1,595/399 training/test instances where the input dimension of each instance is 96. 
We visualize the data distributions for different demographic subgroups in Figure~\ref{fig:data-dist-crime}.

To preprocess the dataset, we first remove the non-predictive attributes and sensitive attributes from the input features. Note that all features in the dataset have already been normalized in $[0, 1]$ so that we do not perform additional normalization to the features. We then replace the missing values with the mean values of the corresponding attributes.

For all methods, we use a two-layer neural network with ReLU as the activation function of the hidden layers and the sigmoid function as the output function for the prediction task (we take the first layer as the feature mapping). The number of neurons in the hidden layers is 50.  We train the neural networks with the \textsc{Adadelta} algorithm with the learning rate 0.1 and a batch size of 256. The models are trained in 100 epochs. For the adversary networks in \textsc{CENet} and \textsc{WassersteinNet}, we use a two-layer neural network with ReLU as the activation function. The number of neurons in the hidden layers of the adversary networks is 100. The adversary network in \textsc{CENet} also uses sigmoid function as the output function. The weight clipping norm in the adversary network of \textsc{WassersteinNet} is 0.002. 

We use the Fairlearn toolkit to implement \textsc{BGL}: we use the exponentiated-gradient algorithm with the default setting as the mitigator and vary the upper bound $\epsilon\in\{0.01, 0.02, 0.03, 0.05\}$ of the bounded group loss constraint. For each value of $\epsilon$, we average the results of ten different random seeds. Note that our experiment setup is different from~\citep{agarwal2019fair}, so our results cannot be directly compared to theirs.

As for \textsc{CoD}, we follow the same hyper-parameter settings as~\citep{komiyama2018nonconvex}: We use the kernelized optimization with the random Fourier features and the RBF kernel (we vary hyper-parameter of the RBF kernel $\gamma\in\{0.1, 1.0, 10, 100\}$) and report the best results with minimal MSE loss for each time we change the fairness budget $\epsilon$. The hyper-parameter settings follow from~\citep{komiyama2018nonconvex}. We also vary $\epsilon\in\{0.01, 0.1, 0.5, 1.0\}$ and average the results of ten different random seeds. Note that our experiment setup is different from~\citep{komiyama2018nonconvex}, so our results cannot be directly compared to theirs.

\begin{figure*}[!ht]
\centering
  \begin{subfigure}[b]{0.45\textwidth}
    \includegraphics[width=\textwidth]{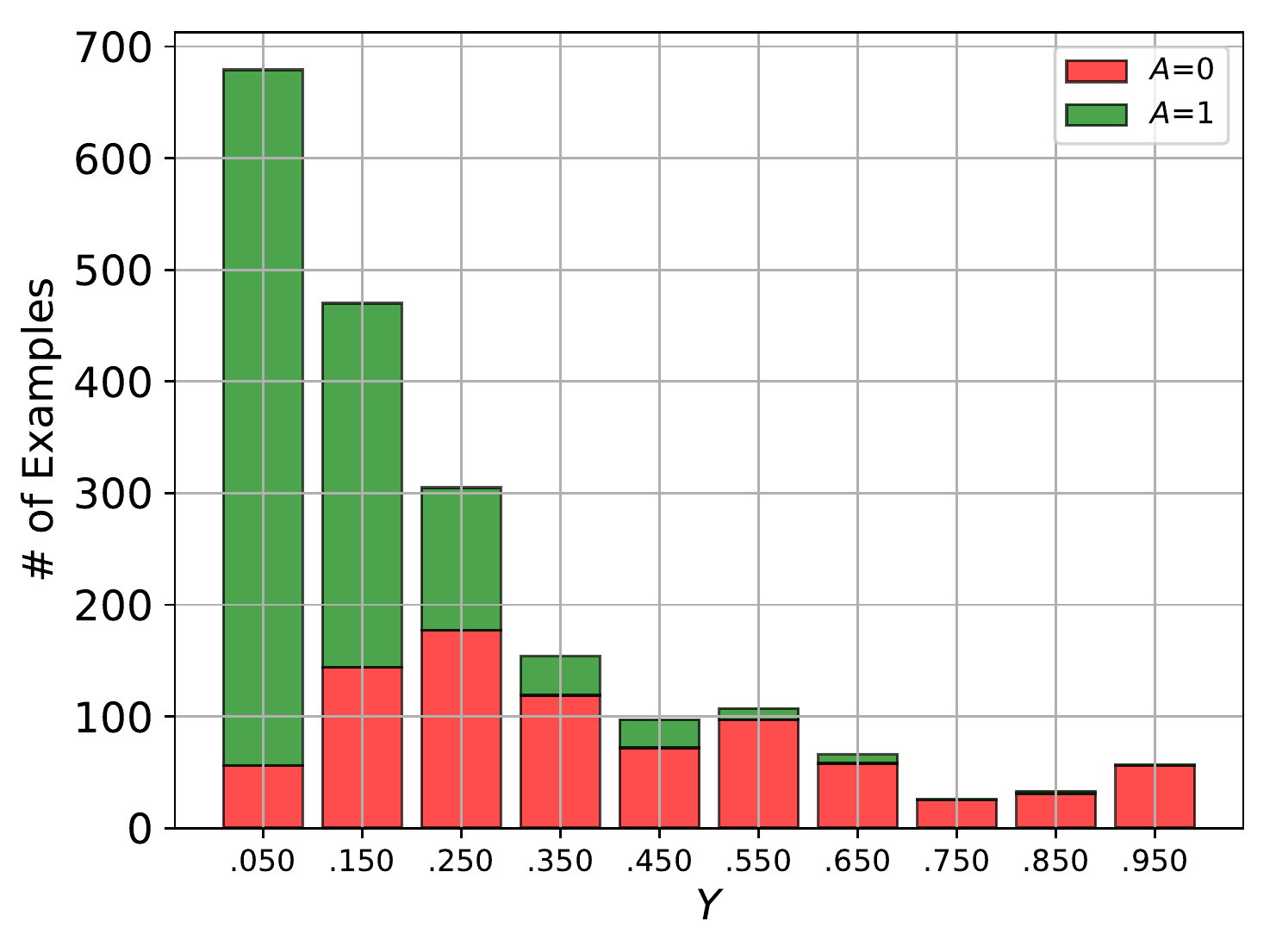}
    \caption{Communities and Crime Dataset}
    \label{fig:data-dist-crime}
  \end{subfigure}
  \begin{subfigure}[b]{0.45\textwidth}
    \includegraphics[width=\textwidth]{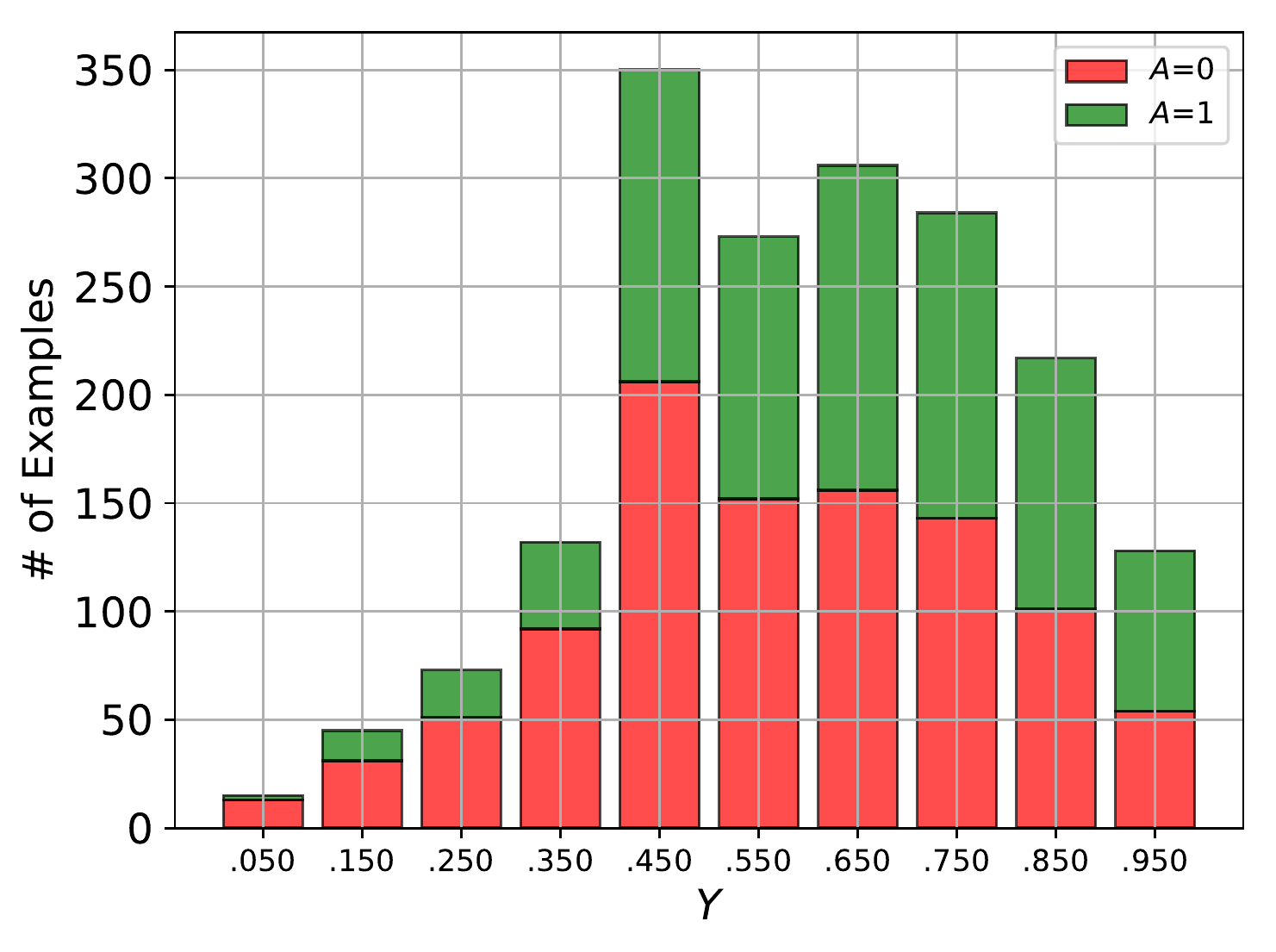}
    \caption{Law School Dataset}
    \label{fig:data-dist-law}
  \end{subfigure}
  \begin{subfigure}[b]{0.45\textwidth}
    \includegraphics[width=\textwidth]{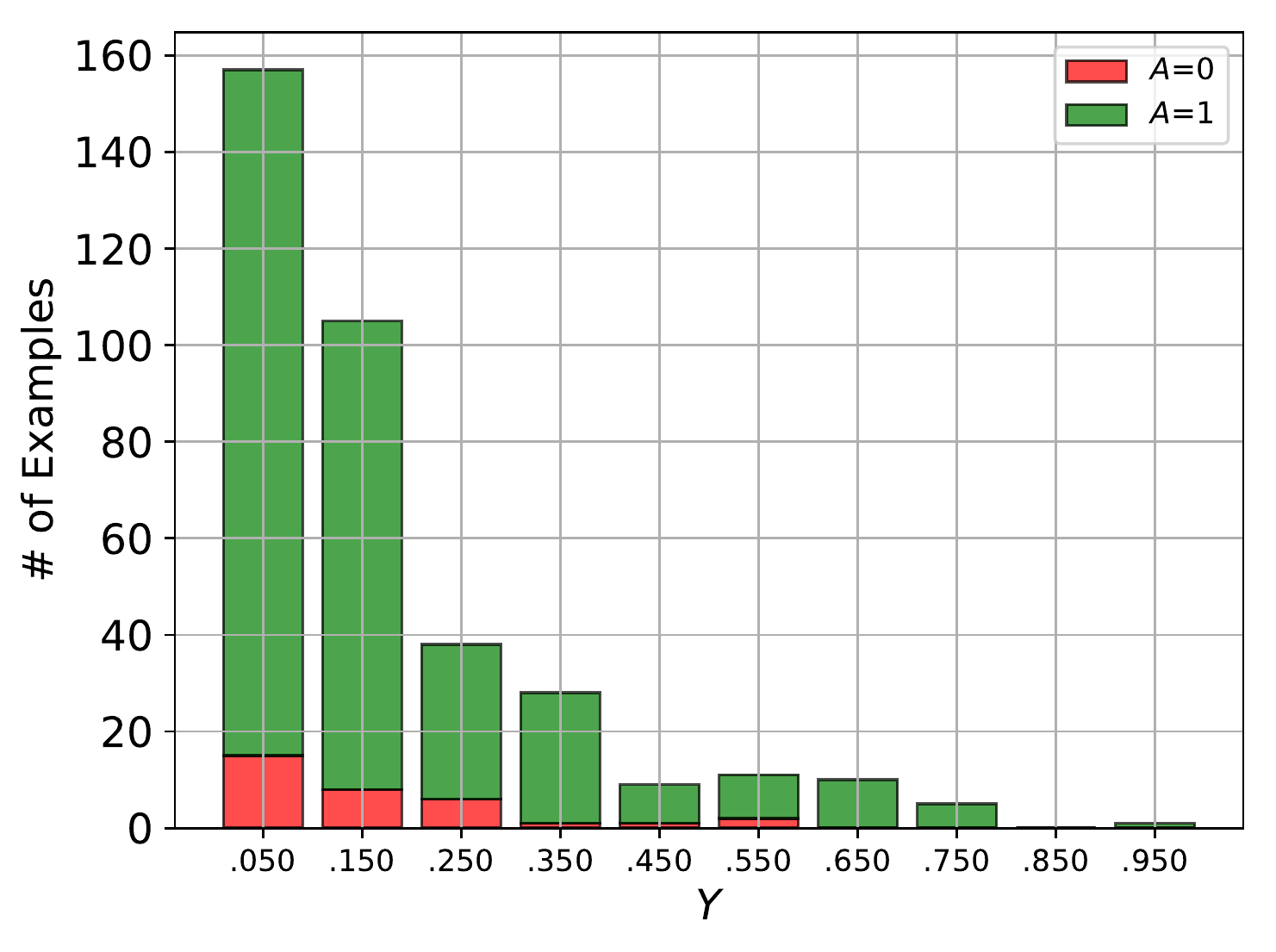}
    \caption{Medical Insurance Cost Dataset}
    \label{fig:data-dist-insurance}
  \end{subfigure}
  \caption{Data distributions for different demographic subgroups in three datasets.}
  \label{fig:data-dist}
\end{figure*}

\paragraph{Law School} The Law School dataset contains 1,823 records for law students who took the bar passage study for Law School Admission\footnote{We use the edited public version of the dataset which can be download here: \url{https://github.com/algowatchpenn/GerryFair/blob/master/dataset/lawschool.csv}}. The features in the dataset include variables such as undergraduate GPA, LSAT score, full-time status, family income, gender, etc. In our experiment, we use gender as the sensitive attribute and undergraduate GPA as the target variable. We split the dataset into train and test sets with the ratio 8/2. We show the data distributions for different demographic subgroups in Figure~\ref{fig:data-dist-law}.

For all methods, we use a two-layer neural network with ReLU as the activation function of the hidden layers and the sigmoid function as the output function for the prediction task (we take the first layer as the feature mapping). The number of neurons in the hidden layers is 10. We train the neural networks with the \textsc{Adadelta} algorithm with the learning rate 0.1 and a batch size of 256. The models are trained in 100 epochs. For the adversary networks in \textsc{CENet} and \textsc{WassersteinNet}, we use a two-layer neural network with ReLU as the activation function. The number of neurons in the hidden layers of the adversary networks is 10. The adversary network in \textsc{CENet} also uses sigmoid function as the output function. The weight clipping norm in the adversary network of \textsc{WassersteinNet} is 0.2. 

We use the Fairlearn toolkit to implement \textsc{BGL}: we use the exponentiated-gradient algorithm with the default setting as the mitigator and vary the upper bound $\epsilon\in\{0.01, 0.02, 0.03, 0.05\}$ of the bounded group loss constraint. For each value of $\epsilon$, we average the results of ten different random seeds. Note that our experiment setup is different from~\citep{agarwal2019fair}, so our results cannot be directly compared to theirs.

As for \textsc{CoD}, we follow the same hyper-parameter settings as~\citep{komiyama2018nonconvex}: We use the kernelized optimization with the random Fourier features and the RBF kernel (we vary hyper-parameter of the RBF kernel $\gamma\in\{0.1, 1.0, 10, 100\}$) and report the best results with minimal MSE loss for each time we change the fairness budget $\epsilon$. The hyper-parameter settings follow from~\citep{komiyama2018nonconvex}. We also vary $\epsilon\in\{0.01, 0.1, 0.5, 1.0\}$ and average the results of ten different random seeds. Note that our experiment setup is different from~\citep{komiyama2018nonconvex}, so our results cannot be directly compared to theirs.

\paragraph{Medical Insurance Cost} The medical insurance cost dataset~\citep{lantz2013machine} is a simulated dataset which was created using real-world demographic statistics from the U.S. Census Bureau.\footnote{We download the public version of data here: \url{https://www.kaggle.com/mirichoi0218/insurance}} The dataset reflect approximately reflect real-world conditions and has been used in the research of regression~\citep{panay2019predicting, hittmeir2019utility, pan2020implicit}. 
It contains 1,338 medical expense examples for patients in the United States, with features such as gender, age, BMI, etc., indicating characteristics of the patient and total annual medical expenses charged to the patients. In our experiment, we use gender as the sensitive attribute and the charged medical expenses as the target variable. In order to reflect the real-world scenarios where the accuracy disparity is significant due to the small and imbalanced dataset, we sub-sample the dataset: we randomly subsample 5\% of examples with gender as male and 50\% of examples with gender as female. After sub-sampling, we get 364 examples in total (33 male examples and 331 female examples). We split the dataset into train and test sets with the ratio 7/3. We visualize the data distributions for different demographic subgroups in Figure~\ref{fig:data-dist-insurance}. 

For all methods, we use a two-layer neural network with ReLU as the activation function of the hidden layers and the sigmoid function as the output function for the prediction task (we take the first layer as the feature mapping). The number of neurons in the hidden layers is 7. We train the neural networks with the \textsc{SGD} algorithm with the learning rate 0.1 and a batch size of 64. The models are trained in 750 epochs. For the adversary networks in \textsc{CENet} and \textsc{WassersteinNet}, we use a two-layer neural network with ReLU as the activation function. The number of neurons in the hidden layers of the adversary networks is 7. The adversary network in \textsc{CENet} also uses sigmoid function as the output function. The weight clipping norm in the adversary network of \textsc{WassersteinNet} is 0.2. 

We use the Fairlearn toolkit to implement \textsc{BGL}: we use the exponentiated-gradient algorithm with the default setting as the mitigator and vary the upper bound $\epsilon\in\{0.01, 0.1, 0.5, 1.0\}$ of the bounded group loss constraint. For each value of $\epsilon$, we average the results of ten different random seeds.

As for \textsc{CoD}, we follow the same hyper-parameter settings as~\citep{komiyama2018nonconvex}: We use the kernelized optimization with the random Fourier features and the RBF kernel (we vary hyper-parameter of the RBF kernel $\gamma\in\{0.1, 1.0, 10, 100\}$) and report the best results with minimal MSE loss for each time we change the fairness budget $\epsilon$. The hyper-parameter settings follow from~\citep{komiyama2018nonconvex}. We also vary $\epsilon\in\{0.01, 0.1, 0.5, 1.0\}$ and average the results of ten different random seeds.

%%%%%%%%%%%%%%%%%%%%%%%%%%%%%%%%%%%%%%%%%%%%%%%%%%%%%%%%%%%%%%%%%%%%%%%%%%%%%%%
%%%%%%%%%%%%%%%%%%%%%%%%%%%%%%%%%%%%%%%%%%%%%%%%%%%%%%%%%%%%%%%%%%%%%%%%%%%%%%%

\section{Additional Experimental Results and Analysis}
\label{app:add-exp-res}

In this section, we provide additional experimental results and analysis.

\subsection{Classification Accuracy vs. Error Gaps in Adult and COMPAS Datasets}

\begin{figure*}[!ht]
\centering
\begin{subfigure}[b]{.40\linewidth}
  \centering
  \includegraphics[width=\linewidth]{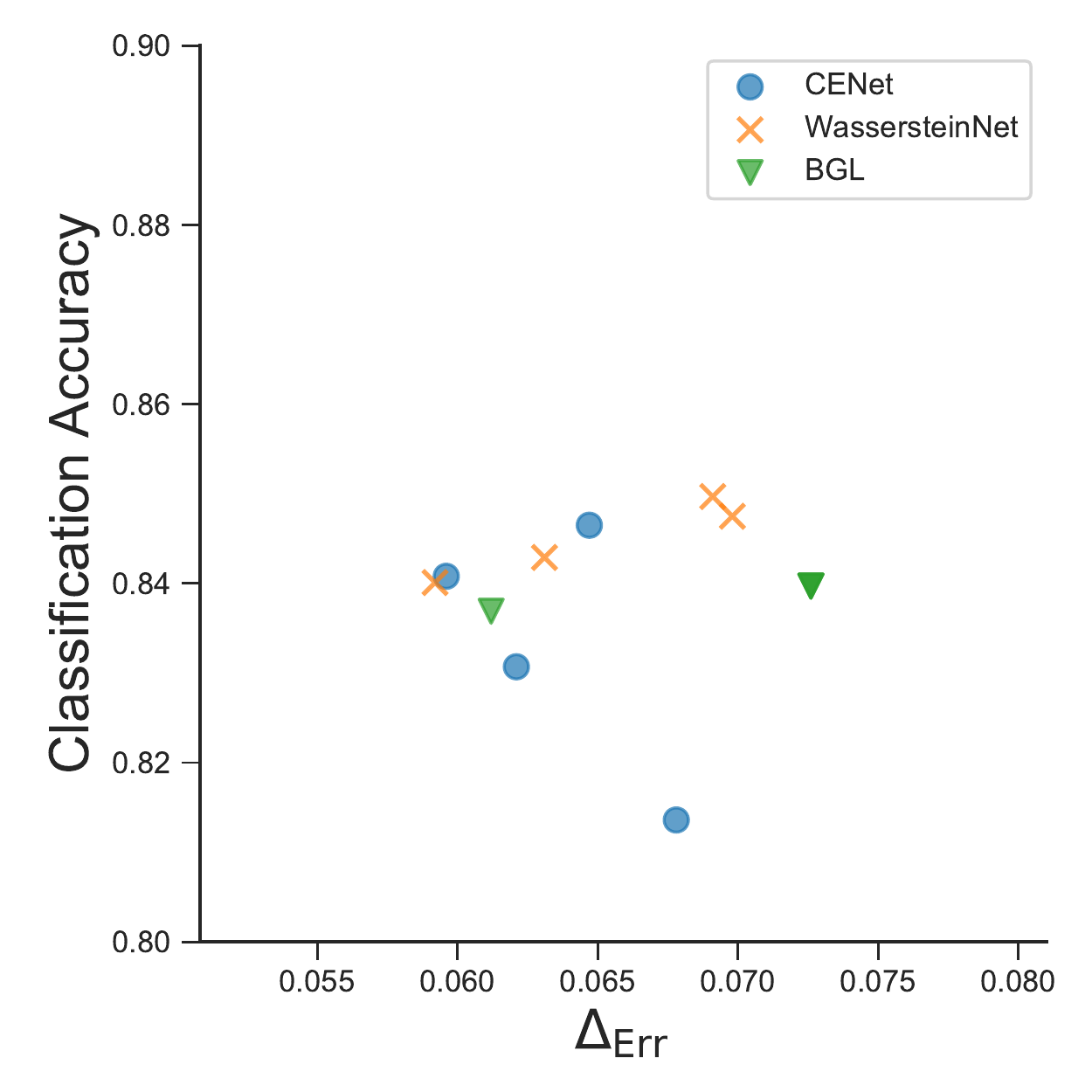}
    \caption{Adult}
  \label{fig:result-cla-adult}
\end{subfigure}
~
\begin{subfigure}[b]{.40\linewidth}
  \centering
  \includegraphics[width=\linewidth]{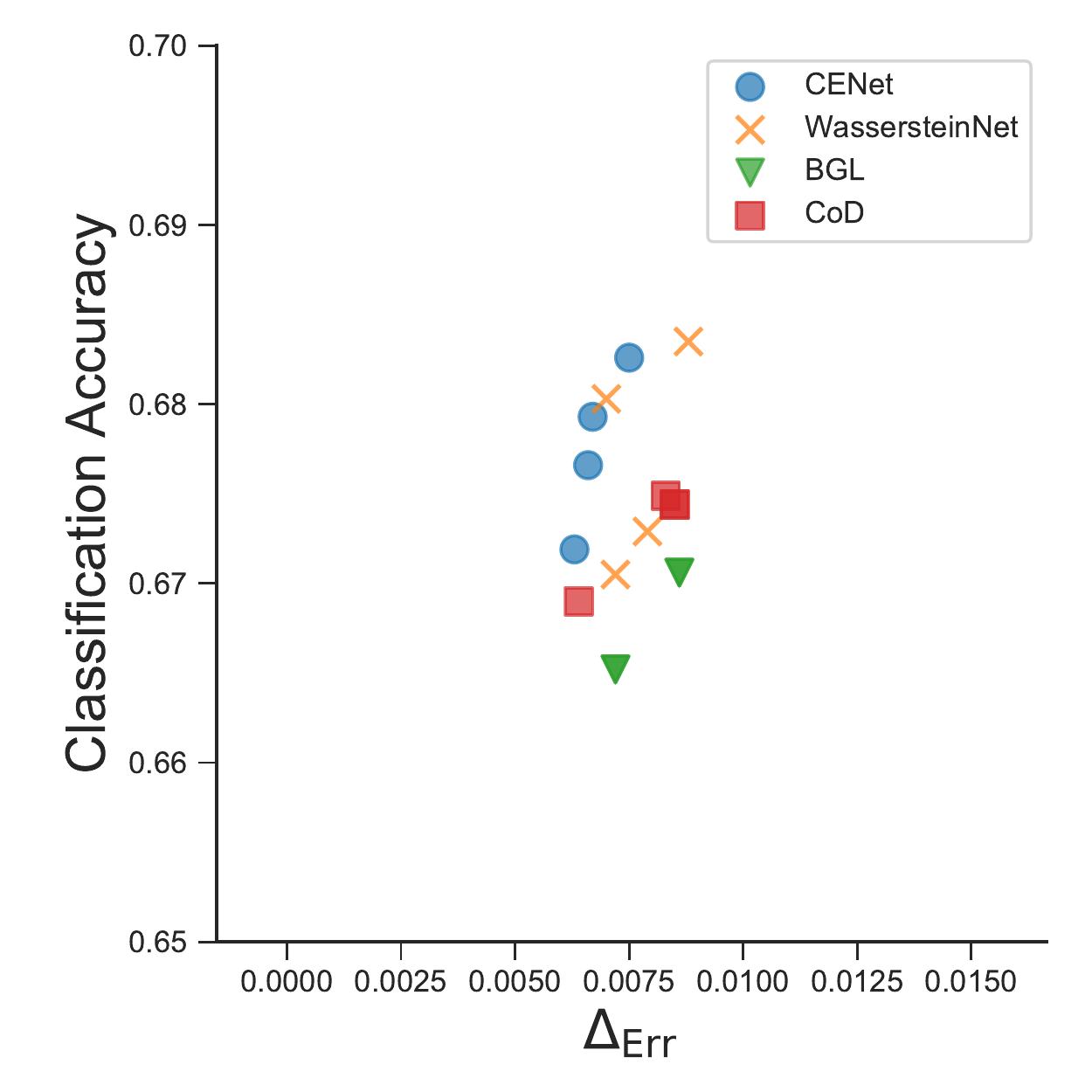}
  \caption{COMPAS}
  \label{fig:result-cla-compas}
\end{subfigure}
%   \vspace*{-1em}
  \caption{Classification accuracy and error gaps of different methods in Adult and COMPAS datasets.}
  \label{fig:classifcation-results}
\end{figure*}

We also report the corresponding classification accuracy for Adult and COMPAS datasets here. In Figure~\ref{fig:classifcation-results}, we can see that our proposed methods achieve the best trade-offs in terms of classification accuracies and error gap values.

\subsection{Impact of Fairness Trade-off in the Baseline Methods}

We present additional experimental results and analyses to gain more insights into how the fairness trade-off parameters (\eg, $\epsilon$) affect the performance of the model predictive performance and accuracy disparity in baseline methods.

\begin{table}[!htb]
\centering
\caption{
$R^2$ regression scores and error gaps when $\epsilon$ changes in \textsc{BGL}.
}
\begin{tabular}{c|lcccc} 
\hline
\multirow{3}{*}{Adult} & $\epsilon$ & 0.1 & 0.2 & 0.3 & 0.5 \\ 
\cline{2-6}
 & $R^2$ & 0.3508 & 0.3696 & 0.3696 & 0.3696 \\ 
 & $\errgap$ & 0.0612 & 0.0726 & 0.0726 & 0.0726 \\ 
\hline
\hline
\multirow{3}{*}{COMPAS} & $\epsilon$ & 0.1 & 0.2 & 0.3 & 0.5 \\ 
\cline{2-6}
 & $R^2$ & 0.1478 & 0.1478 & 0.1507 & 0.1507 \\ 
\cline{2-6}
 & $\errgap$ & 0.0072 & 0.0072 & 0.0086 & 0.0086 \\ 
\hline
\hline
\multirow{3}{*}{Crime} & $\epsilon$ & 0.01 & 0.02 & 0.03 & 0.05 \\ 
\cline{2-6}
 & $R^2$ & 0.3922 & 0.3922 & 0.5380 & 0.5380 \\ 
\cline{2-6}
 & $\errgap$ & 0.0189 & 0.0189 & 0.0238 & 0.0238 \\ 
\hline
\hline
\multirow{3}{*}{Law} & $\epsilon$ & 0.01 & 0.02 & 0.03 & 0.05 \\ 
\cline{2-6}
 & $R^2$ & 0.1407 & 0.1407 & 0.1407 & 0.1412 \\ 
\cline{2-6}
 & $\errgap$ & 0.0094 & 0.0094 & 0.0094 & 0.0101 \\
\hline
\hline
\multirow{3}{*}{Insurance} & $\epsilon$ & 0.0001 & 0.01 & 0.05 & 0.1 \\ 
\cline{2-6}
 & $R^2$ & 0.6804 & 0.6855 & 0.6855 & 0.6855 \\ 
\cline{2-6}
 & $\errgap$ & 0.0145 & 0.0144 & 0.0144 & 0.0144 \\
\hline
\end{tabular}
\label{tab:epsilon-bgl-result}
\end{table}

Table~\ref{tab:epsilon-bgl-result} shows $R^2$ regression scores and error gaps when $\epsilon$ changes in \textsc{BGL}. We see that with the decrease of the trade-off parameter $\epsilon$, both the values of $R^2$ and error gaps decrease. This is because when the upper bound of $\epsilon$ in \textsc{BGL} is small, the accuracy disparity is also mitigated. When $\epsilon$ is above/below a certain threshold, $R^2$ scores and error gap values then increase/decrease. 
% It is also worth to note that the exponentiated-gradient approach to solve \textsc{BGL} does not introduce the randomness during optimization.

\begin{table}[!htb]
\centering
\caption{
$R^2$ regression scores and error gaps when $\epsilon$ changes in \textsc{CoD}.
}
\begin{tabular}{c|lcccc} 
\hline
\multirow{3}{*}{COMPAS} & $\epsilon$ & 0.01 & 0.1 & 0.5 & 1.0 \\ 
\cline{2-6}
 & $R^2$ & 0.1033 & 0.1144 & 0.1146 & 0.1146\\ 
\cline{2-6}
 & $\errgap$ & 0.0064 & 0.0083 & 0.0085 & 0.0085 \\ 
\hline
\hline
\multirow{3}{*}{Crime} & $\epsilon$ & 0.01 & 0.1 & 0.5 & 1.0 \\ 
\cline{2-6}
 & $R^2$ & 0.1262 & 0.3284 & 0.3603 & 0.3603 \\ 
\cline{2-6}
 & $\errgap$ & 0.0312 & 0.0307 & 0.0343 & 0.0343 \\ 
\hline
\hline
\multirow{3}{*}{Law} & $\epsilon$ & 0.01 & 0.1 & 0.5 & 1.0 \\ 
\cline{2-6}
 & $R^2$ & 0.1262 & 0.3284 & 0.3606 & 0.3603 \\ 
\cline{2-6}
 & $\errgap$ & 0.0312 & 0.0307 & 0.0343 & 0.0343 \\
\hline
\hline
\multirow{3}{*}{Insurance} & $\epsilon$ & 0.01 & 0.1 & 0.5 & 1.0 \\ 
\cline{2-6}
 & $R^2$ & 0.2711 & 0.2691 & 0.2689 & 0.2689 \\ 
\cline{2-6}
 & $\errgap$ & 0.0203 & 0.0210 & 0.0211 & 0.0211 \\
\hline
\end{tabular}
\label{tab:epsilon-cod-result}
\end{table}

Table~\ref{tab:epsilon-cod-result} shows $R^2$ regression scores and error gaps when $\epsilon$ changes in \textsc{CoD}. We see that with the decrease of the trade-off parameter $\epsilon$, both the values of $R^2$ and error gaps decrease in general. 
% It is worth to note that the the optimization of QCQP to solve \textsc{CoD} does not introduce the randomness, and the only randomness introduced in COMPAS dataset is because using the random Fourier features in prediction achieves the best performance in COMPAS dataset. 

\subsection{Visualization of Training Processes}

We visualize the training processes of our proposed methods \textsc{CENet} and \textsc{WassersteinNet} in the Adult dataset and COMPAS dataset in Figure~\ref{fig:adult-training-vis} and Figure~\ref{fig:compas-training-vis}, respectively. We also compare their training dynamics with the model performance when we solely minimize the MSE loss (\ie, $\lambda=0$) and we term it as \textsc{No Debias}. 

\begin{figure}[!htb]
\centering
\vspace{-4mm}
\begin{subfigure}[b]{.48\linewidth}
  \centering
  \includegraphics[width=\linewidth]{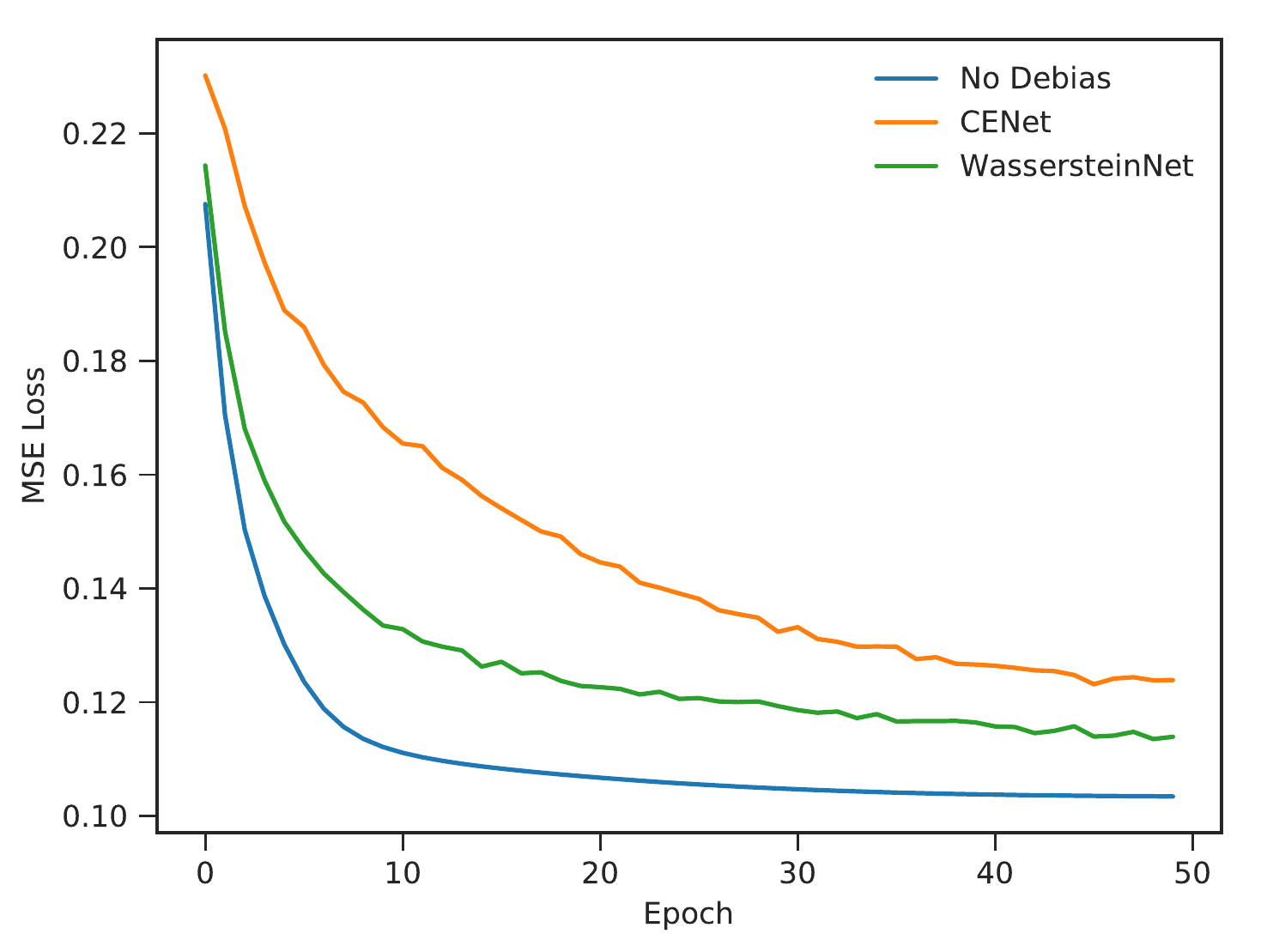}
    \caption{MSE Loss}
  \label{fig:adult-mse}
\end{subfigure}
~
\begin{subfigure}[b]{.48\linewidth}
  \centering
  \includegraphics[width=\linewidth]{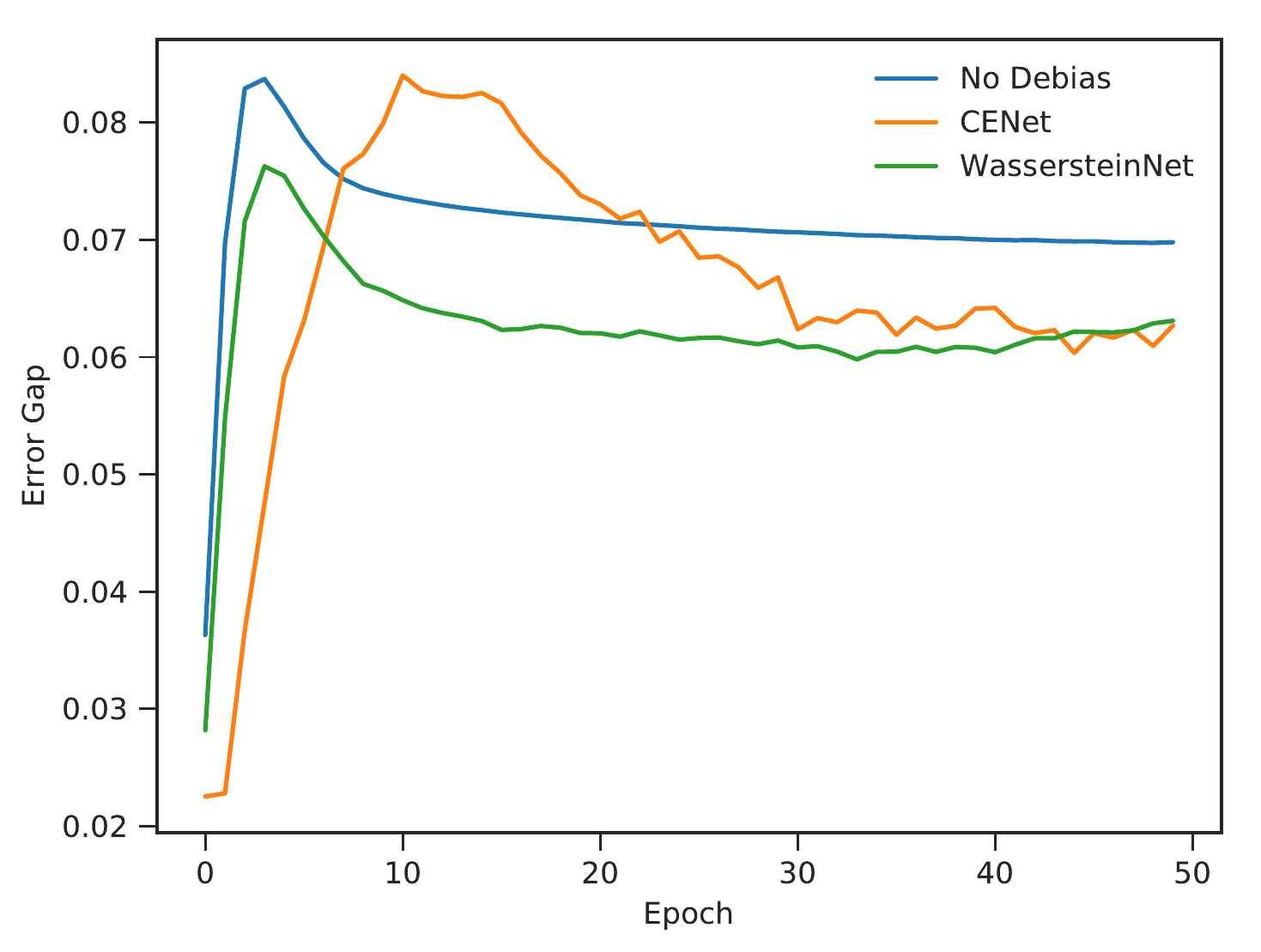}
  \caption{Error Gap}
  \label{fig:adult-errgap}
\end{subfigure}
\caption{Training visualization of \textsc{CENet}, \textsc{WassersteinNet} ($\lambda=50$) and \textsc{No Debias} ($\lambda=0$) in the Adult dataset.}
%   \vspace{-4mm}
 \label{fig:adult-training-vis}
\end{figure}

\begin{figure}[!htb]
\centering
% \vspace{-4mm}
\begin{subfigure}[b]{.48\linewidth}
  \centering
  \includegraphics[width=\linewidth]{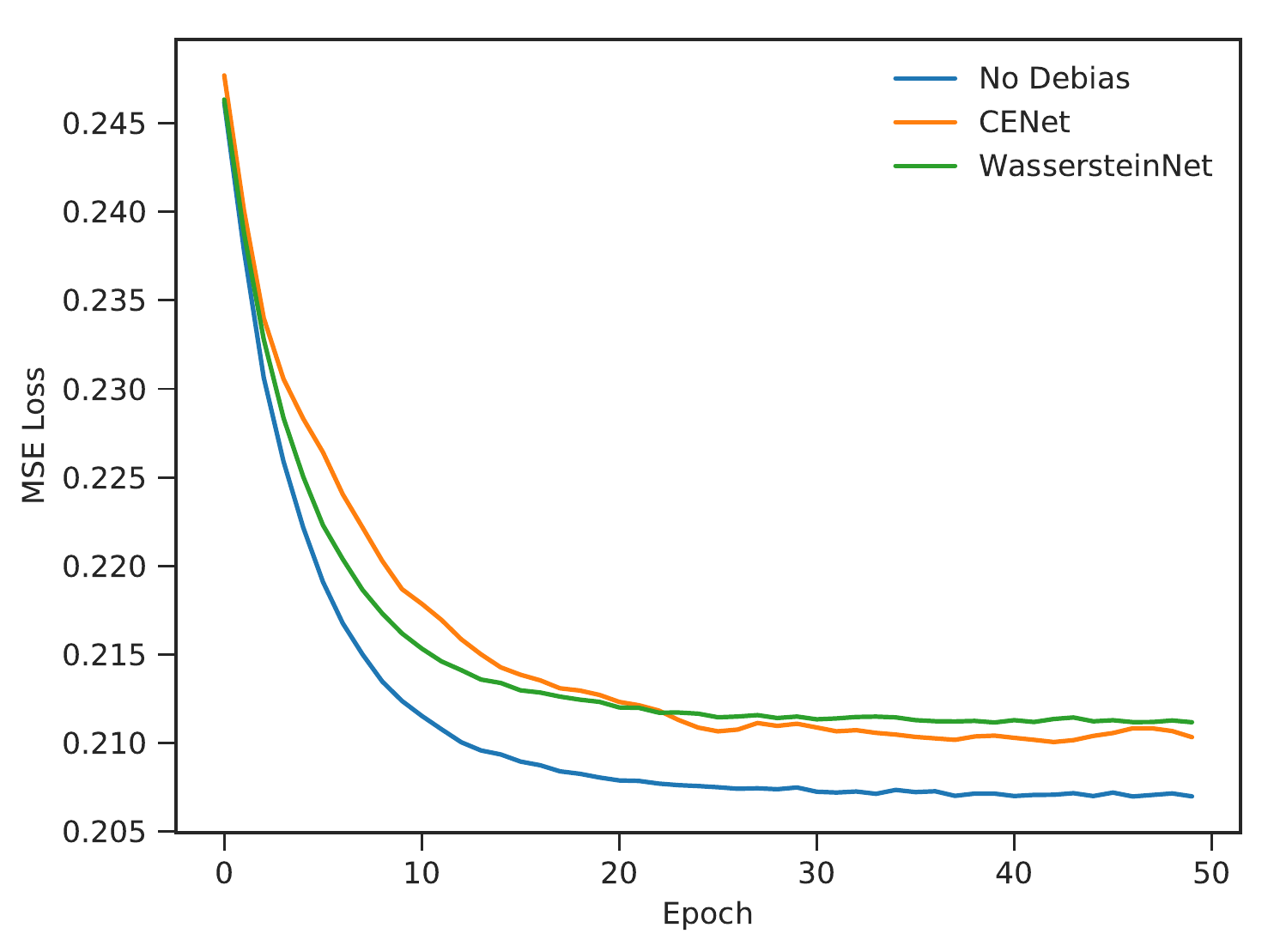}
    \caption{MSE Loss}
  \label{fig:compas-mse}
\end{subfigure}
~
\begin{subfigure}[b]{.48\linewidth}
  \centering
  \includegraphics[width=\linewidth]{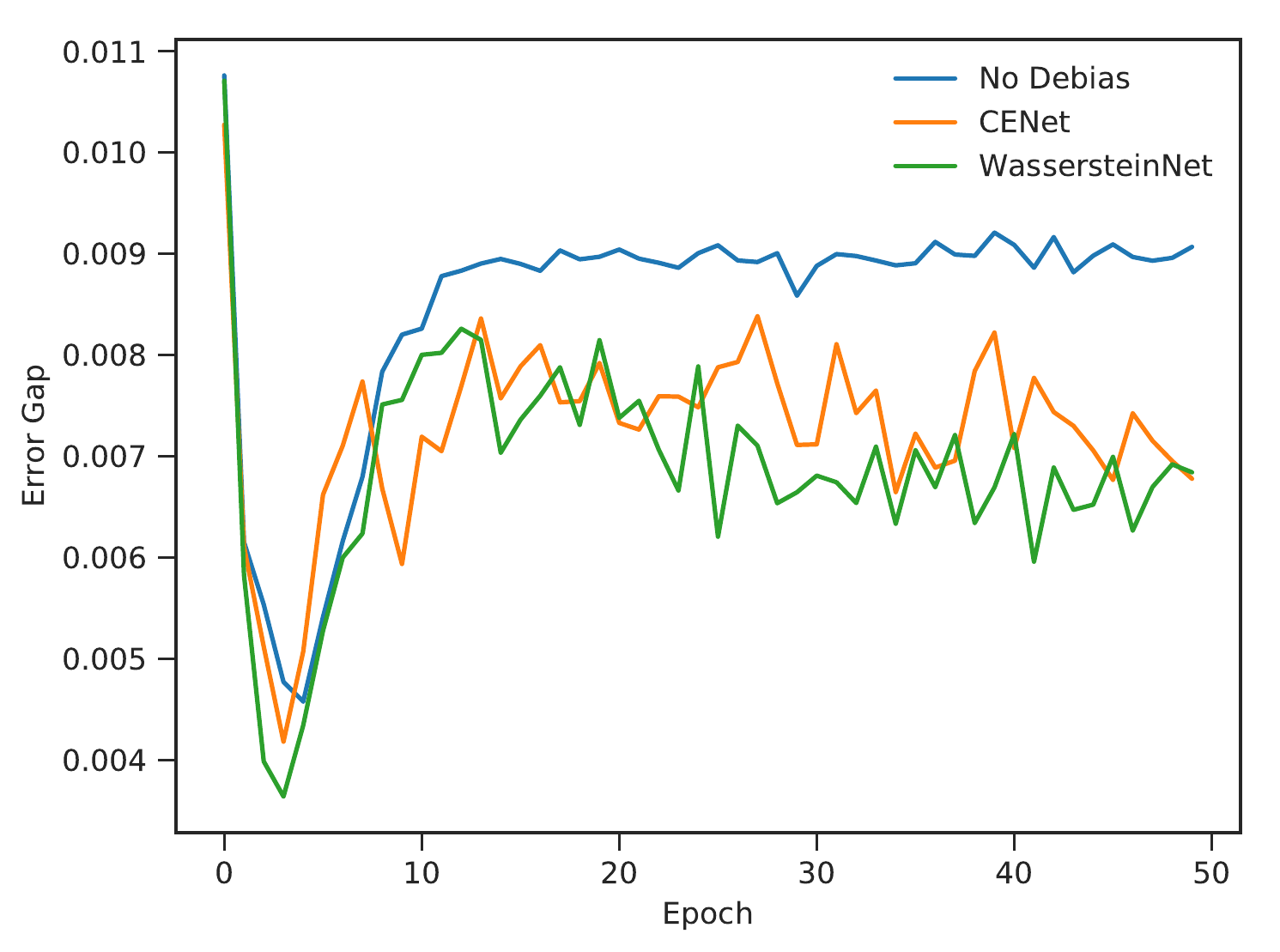}
  \caption{Error Gap}
  \label{fig:compas-errgap}
\end{subfigure}
\caption{Training visualization of \textsc{CENet}, \textsc{WassersteinNet} ($\lambda=5$) and \textsc{No Debias} ($\lambda=0$) in the COMPAS dataset.}
  \vspace{-4mm}
  \label{fig:compas-training-vis}
\end{figure}

In Figure~\ref{fig:adult-training-vis} and Figure~\ref{fig:compas-training-vis}, we can see that as the training progresses go on, the MSE losses in both datasets are decreasing and finally converge. However, the training dynamics of error gaps are much more complex even in the \textsc{No Debias} case. Before convergence, the training dynamics of error gaps differs among different datasets. Our methods enforce the models to converge to the points where error gap are smaller while preserving the models' predictive performance. It is also worth to note that minimax optimization makes the training processes somehow unstable, especially when training \textsc{CENet}.

%%%%%%%%%%%%%%%%%%%%%%%%%%%%%%%%%%%%%%%%%%%%%%%%%%%%%%%%%%%%%%%%%%%%%%%%%%%%%%%
%%%%%%%%%%%%%%%%%%%%%%%%%%%%%%%%%%%%%%%%%%%%%%%%%%%%%%%%%%%%%%%%%%%%%%%%%%%%%%%

\end{document}